\documentclass[letterpaper]{article} % DO NOT CHANGE THIS
\usepackage{aaai2026}  % DO NOT CHANGE THIS
\usepackage{times}  % DO NOT CHANGE THIS
\usepackage{helvet}  % DO NOT CHANGE THIS
\usepackage{courier}  % DO NOT CHANGE THIS
\usepackage[hyphens]{url}  % DO NOT CHANGE THIS
\usepackage{graphicx} % DO NOT CHANGE THIS
\usepackage{natbib}  % DO NOT CHANGE THIS AND DO NOT ADD ANY OPTIONS TO IT
\usepackage{caption} % DO NOT CHANGE THIS AND DO NOT ADD ANY OPTIONS TO IT
\usepackage{algorithm}
\usepackage{algorithmic}
\usepackage{amsthm}
\usepackage{amsmath}
\usepackage{amssymb}
\usepackage{caption} % DO NOT CHANGE THIS AND DO NOT ADD ANY OPTIONS TO IT

\usepackage{booktabs}
\usepackage{multicol}
\usepackage{multirow}
\usepackage{tcolorbox}
\usepackage{colortbl}

\newtheorem{thm}{Theorem}

\usepackage{algorithm}
\usepackage{algorithmic}
\usepackage{subcaption}  % For subfigures

\usepackage{float}
\usepackage{caption}
\usepackage{tcolorbox}
\usepackage{listings}
\tcbuselibrary{skins}

\usepackage{changepage}
\usepackage{fancyvrb}
\usepackage{xcolor}
\usepackage{mdframed}
\usepackage{sansmath}

\usepackage{newfloat}
\usepackage{listings}
\DeclareCaptionStyle{ruled}{labelfont=normalfont,labelsep=colon,strut=off} % DO NOT CHANGE THIS
\floatstyle{ruled}
\newfloat{listing}{tb}{lst}{}
\floatname{listing}{Listing}
\title{EoH-S: Evolution of Heuristic Set using LLMs for Automated Heuristic Design}
\author {
    Fei Liu\textsuperscript{\rm 1},
    Yilu Liu\textsuperscript{\rm 1},
    Qingfu Zhang\textsuperscript{\rm 1},
    Xialiang Tong\textsuperscript{\rm 2},
    Mingxuan Yuan\textsuperscript{\rm 2}
}
\affiliations {
    \textsuperscript{\rm 1}Department of Computer Science, City University of Hong Kong\\
    \textsuperscript{\rm 2}Huawei Noah’s Ark Lab\\
    fliu36-c@my.cityu.edu.hk, qingfu.zhang@cityu.edu.hk
}

\begin{document}

\maketitle

\begin{abstract}
Automated Heuristic Design (AHD) using Large Language Models (LLMs) has achieved notable success in recent years. Despite the effectiveness of existing approaches, they only design a single heuristic to serve all problem instances, often inducing poor generalization across different distributions or settings. To address this issue, we propose Automated Heuristic Set Design (AHSD), a new formulation for LLM-driven AHD. The aim of AHSD is to automatically generate a small-sized complementary heuristic set to serve diverse problem instances, such that each problem instance could be optimized by at least one heuristic in this set. We show that the objective function of AHSD is monotone and supermodular. Then, we propose \underline{E}volution of \underline{H}euristic \underline{S}et (EoH-S) to apply the AHSD formulation for LLM-driven AHD. With two novel mechanisms of complementary population management and complementary-aware memetic search, EoH-S could effectively generate a set of high-quality and complementary heuristics. Comprehensive experimental results on three AHD tasks with diverse instances spanning various sizes and distributions demonstrate that EoH-S consistently outperforms existing state-of-the-art AHD methods and achieves up to 60\% performance improvements. 
\end{abstract}

%标题：Evolving a set of heuristics for Generalizable AAD
%图，圈出最后一个
%重点在set上
%看看set中每个算法的行为或思路差异
%selection用llm刻画？

\section{Introduction}

Automated Heuristic Design (AHD) using Large Language Models (LLMs) has been a success in recent years~\cite{liu2024systematic}. With the code generation and language comprehension capability of LLMs, the automation and flexibility of algorithm design have been significantly enhanced. Until now, LLM-driven AHD has found successful applications across diverse domains, including optimization~\cite{liu2024evolution,ye2024reevo,van2024llamea,yao2024evolve,ye2025large,dat2025hsevo,li2025llm}, mathematics~\cite{romera2024mathematical,novikov2025alphaevolve}, and machine learning~\cite{mo2025autosgnn}. 

A prevalent paradigm in LLM-driven AHD is to integrate LLMs as heuristic designers within certain iterative search frameworks~\cite{zhang2024understanding}, including evolutionary search~\cite{liu2024evolution,ye2024reevo,van2024llamea}, neighborhood search~\cite{xie2025llm}, and Monte Carlo tree search (MCTS)~\cite{zheng2025monte}. Notable examples include EoH~\cite{liu2024evolution}, which evolves both thoughts and codes for effective automated heuristic design. FunSearch~\cite{romera2024mathematical} utilizes a multi-island evolutionary framework with a single prompt strategy for guiding LLMs in function search. Moreover, ReEvo~\cite{ye2024reevo} integrates short- and long-term reflection strategies to enhance the heuristic design process. 

Despite these advancements, existing methods mainly focus on identifying a single heuristic with the best average performance across a set of training instances. This approach may suffer from generalization limitations due to the following two reasons: 
1) finding a single heuristic with the best performance across all diverse instances is inherently difficult~\cite{sim2025beyond}, and 2) even when a heuristic excels on most training instances, it often fails to generalize to unseen test instances with different distributions or scales~\cite{shi2025generalizable}.

A common approach to tackle the generalization limitations of heuristic design is to use an algorithm portfolio~\cite{gomes2001algorithm,tang2021few} (i.e., a combination of different algorithms). It is natural for us to leverage this approach to deal with the challenges faced by LLM-driven AHD: using different heuristics rather than a single heuristic. However, given the typically enormous number of problem instances, it is impractical, if not impossible, to find one heuristic for each instance. 
With these concerns, this paper makes the following contributions:

\begin{itemize}
    \item We introduce Automated Heuristic Set Design (AHSD), a new formulation for LLM-driven AHD. AHSD aims to automatically generate a small-sized complementary heuristic set to serve diverse problem instances, such that each problem instance could be optimized by at least one heuristic in this set. We show that the objective function of AHSD is monotone and supermodular.
    \item We propose Evolution of Heuristic Set (EoH-S) to apply the AHSD formulation for LLM-driven AHD. EoH incorporates two key components of complementary population management and diversity-aware memetic search to effectively search a set of high-quality and complementary heuristics.
    \item We conduct extensive experiments on three AHD tasks with instances of different distributions and sizes, as well as on many benchmark sets. Results show that EoH-S excels in all scenarios and achieves up to 60\% performance improvements over state-of-the-art AHD methods. 
\end{itemize}

\section{Automated Heuristic Set Design (AHSD)}

\subsection{Problem Formulation}

Suppose we are given a target heuristic design task $\mathcal{T}$ (e.g., Traveling Salesman Problem (TSP)) with $m$ diverse problem instances $I=\{i_1,\dots,i_m\}$, the aim of AHSD is to automatically generate a heuristic set $H=\{h_1,\dots,h_k\}\subseteq D$, where $1<k\ll m$ and $D$ is the search space, such that each problem instance could be optimized by at least one heuristic in this set, which could be mathematically expressed as
\begin{equation}
\label{ahsd}
    \min_{H\subseteq D}\left(f^*_{H,i_1}, f^*_{H,i_2},\dots,f^*_{H,i_m}\right),
\end{equation}
where $f^*_{H,i}=\min_{h\in H} f_i(h)$, and $f_i(h)$ denotes the performance score of heuristic $h$ on instance $i$ (lower is better). Since (\ref{ahsd}) evaluates $H$ from multiple criteria, it is difficult to directly optimize it. To address this issue, we use linear scalarization~\cite{triantaphyllou2000multi,zhang2021survey} to transform (\ref{ahsd}) into the following optimization objective:
% we use linear scalarization~\cite{triantaphyllou2000multi,zhang2021survey}, a common approach for aggregating multiple optimization criteria, to transform (\ref{ahsd}) as the following optimization objective:

\begin{equation}
\label{obj}
    \mathcal{F}(H) = \frac{1}{m} \sum_{i \in I} f^*_{H,i}. 
\end{equation}
For clarity, we call $\mathcal{F}(H)$ as \textbf{Complementary Performance Index (CPI)}. Taking CPI as the objective function, the AHSD problem could be formally stated as below.

\begin{tcolorbox}
\textit{AHSD Problem}:
Given a target heuristic design task $\mathcal{T}$ with $m$ task instances $I=\{i_1,\dots,i_m\}$, the aim of AHSD is to automatically generate a heuristic set $H=\{h_1,\dots,h_k\}\subseteq D$ with $1<k\ll m$ such that $\mathcal{F}(H)$ is minimized.
\end{tcolorbox}
When $k = 1$, the AHSD problem is equal to finding a single optimal heuristic with the best average performance across all instances, which is aligned with the objective of most LLM-driven AHD methods.
When $k = m$, the AHSD problem is equal to finding the best heuristic for each instance independently. Thus, we consider the non-trivial case of $1<k\ll m$. 
Next, we demonstrate some theoretical properties of $\mathcal{F}(H)$.

\subsection{Theoretical Analysis}

\begin{thm}[Optimization Complexity]
\label{th1}
The optimization of $\mathcal{F}(H)$ is NP-hard.
\end{thm}
\begin{proof}
We consider an instance of $\mathcal{F}(H)$ by setting $f_i(h)$ as
\begin{equation}
f_i(h) = \text{dis}(h_i^*,h),
\end{equation} 
where $h_i^*\in D$ is the optimal heuristic for instance $i$ and $\text{dis}(h_i^*,h)$ denotes a certain distance between $h_i^*$ and $h$. Thus, the objective function of this instance is
\begin{equation}
\mathcal{F}'(H) = \frac{1}{m}\sum_{i\in I}{\min_{h\in H}\text{dis}(h_i^*,h)}.
\end{equation} 
We can find that the aim of this instance is equal to finding $k$ centers $\left\{h_1,\dots,h_k\right\}$ for $m$ data points $\left\{h_1^*,\dots,h_m^*\right\}$, such that each data point can be assigned to the optimal center, which is consistent with the aim of the Discrete Clustering Problem (DCP)~\cite{dcp}. Since DCP is NP-hard when $k>1$~\cite{dcp}, the optimization of $\mathcal{F}(H)$ is also NP-hard, which concludes the proof of Theorem~\ref{th1}.
\end{proof}
%\textit{Proof Sketch.} This theorem can be proved by reducing $\mathcal{F}(H)$ to the objective function of the classic NP-hard Discrete Clustering Problem (DCP)~\cite{dcp}.

\begin{thm}[Monotonicity and Supermodularity]
\label{th2}
For any two heuristic sets $U \subseteq V \subseteq D$ with $1\leq |U|\leq |V|$, $\mathcal{F}(U) \geq \mathcal{F}(V)$ holds. Besides, for any heuristic $h' \in D$\textbackslash $V$, $\mathcal{F}(U) - \mathcal{F}(U \cup \left\{h' \right\}) \geq \mathcal{F}(V) - \mathcal{F}(V \cup \left\{h' \right\})$ holds.
\end{thm}
\begin{proof}
We first demonstrate that $f_{U,i}^* \geq f_{V,i}^*$ holds for all instances $i\in\left\{i_1,\dots,i_m\right\}$. For simplicity, we use $j^*$ to denote the index of the best heuristic for $i$ within $V$, i.e., $j^*=argmin_{1 \leq j \leq |V|}f_i(h_j)$. Therefore, we only need to discuss the following two cases: 
\begin{enumerate}
\item $1 \leq j^* \leq |U|$. This case means that the best heuristic for $i$ is in $U$. As $U \subseteq V$, we can derive $f_{U,i}^* = f_{V,i}^*$.
\item $|U| < j^* \leq |V|$. This case means that the best heuristic for $i$ is in $V$\textbackslash $U$, showing that $f_{U,i}^* \geq f_{V,i}^*$.
\end{enumerate}
These two cases indicate $f_{U,i}^* \geq f_{V,i}^*$. As $\mathcal{F}(H)$ is a convex combination of $f_{H,i}^*$, we have $\mathcal{F}(U) \geq \mathcal{F}(V)$, showing that $\mathcal{F}(H)$ is monotone.

Let $U'=U \cup \left\{h'\right\}$. We then show that $f_{H,i}^*$ is supermodular, i.e., $f_{U,i}^*-f_{U',i}^*\geq f_{V,i}^*-f_{V',i}^*$. For simplicity, we use $j^*$ to denote the index of the best heuristic for $i$ within $V'$, i.e., $j^*=argmin_{1 \leq j \leq |V|+1}f_i(h_j)$. Therefore, we only need to discuss the following three cases: 
\begin{enumerate}
\item $1 \leq j^* \leq |U|$. This case means that the best heuristic for $i$ is in $U$. Thus, we can derive $f_{U,i}^* = f_{U',i}^* = f_{V,i}^* = f_{V',i}^*$, showing that $f_{U,i}^*-f_{U',i}^*\geq f_{V,i}^*-f_{V',i}^*$.
\item $|U| < j^* \leq |V|$. This case means that the best heuristic for $i$ is in $V$\textbackslash $U$, showing that $f_{V,i}^* = f_{V',i}^*$. Since $f_{U,i}^*\geq f_{U',i}^*$, we have $f_{U,i}^*-f_{U',i}^*\geq f_{V,i}^*-f_{V',i}^*$.
\item $j^* = |V|+1$. This case means that the best heuristic for $i$ is $\boldsymbol{x}'$, showing that $f_{U',i}^* = f_{V',i}^*$. Since $f_{U,i}^*\geq f_{V,i}^*$, we have $f_{U,i}^*-f_{U',i}^*\geq f_{V,i}^*-f_{V',i}^*$.
\end{enumerate}
These three cases confirm the supermodularity of $f_{H,i}^*$. Similarly, as $\mathcal{F}(H)$ is a convex combination of $f_{H,i}^*$, $\mathcal{F}(U) - \mathcal{F}(U \cup \left\{h' \right\}) \geq \mathcal{F}(V) - \mathcal{F}(V \cup \left\{h' \right\})$ holds, showing that $\mathcal{F}(H)$ is supermodular. 
Thus, Theorem~\ref{th2} holds.
\end{proof}
%\textit{Proof Sketch.} This theorem can be proved by showing the monotonicity and supermodularity of $f^*_{H,i}$, as $\mathcal{F}(H)$ is a convex combination of $f^*_{H,i}$. 

Given the optimization complexity of $\mathcal{F}(H)$, evolutionary search is still a useful optimization framework. As evolution often operates on a heuristic population, and $\mathcal{F}(H)$ is monotone and supermodular, prior work~\cite{ga} has shown that the Greedy Algorithm (GA) could always identify a small-sized complementary heuristic set from such a population while providing the following theoretical performance guarantee. 
%\fei{one more sentence on method}

\begin{thm}[Theoretical Performance Guarantee of GA]
\label{th3}
Let $P=\left\{h_1,\dots,h_n\right\}\subset D$ be a heuristic population with $n>k$ and $H^o$ be the optimal heuristic set within $P$ for solving $\mathcal{F}(H)$. Then, the output of GA, referred to as $H^{ga}$, satisfies $\mathcal{F}(H_1) - \mathcal{F}(H^{ga}) \geq (1-k/(ek-e))(\mathcal{F}(H_1) - \mathcal{F}(H^o))$, where $H_1=\left\{argmin_{h \in P}\mathcal{F}(\left\{h\right\})\right\}$ and $e$ is the natural constant.
\end{thm}

\begin{figure*}[thbp]
    \centering
    \includegraphics[width=0.98\textwidth]{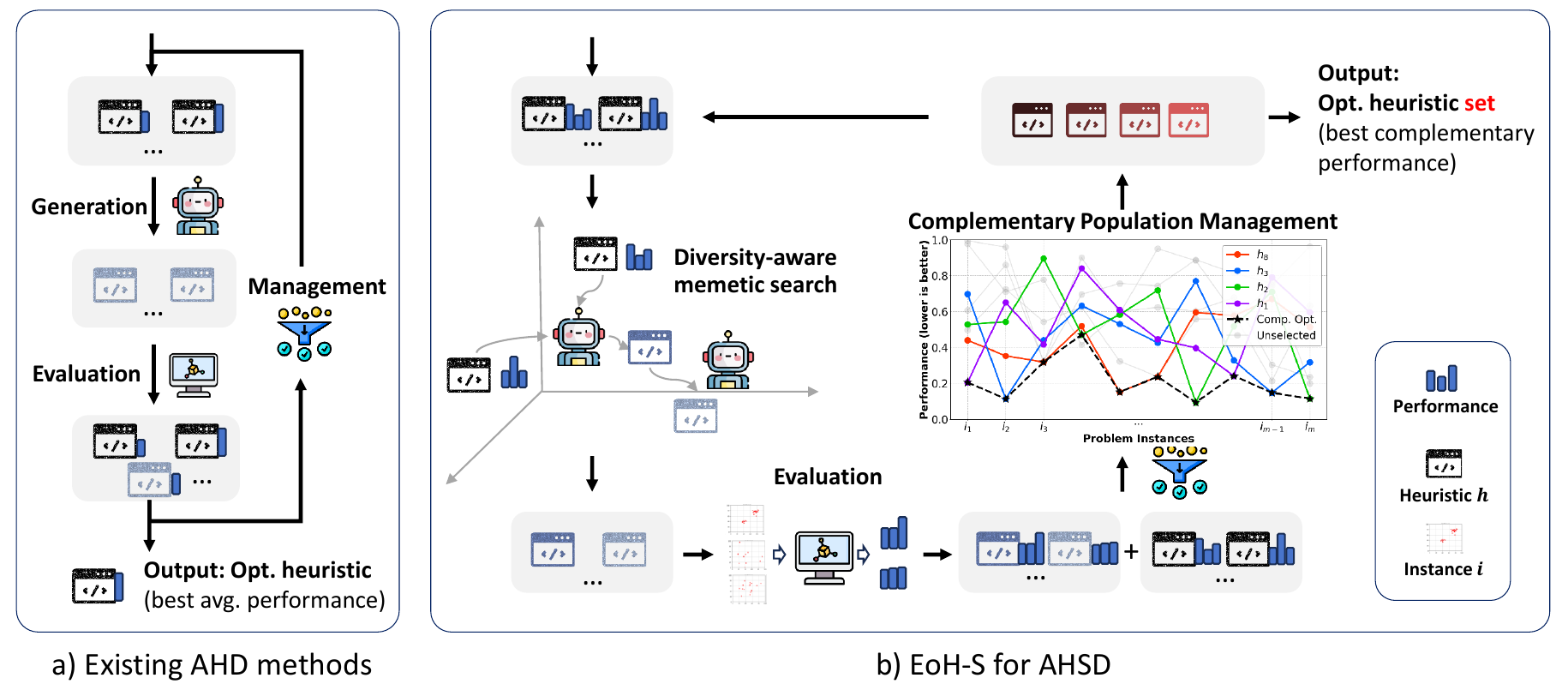}
    \caption{\textbf{a) Existing LLM-driven Automated Heuristic Design (AHD) methods} employ an iterative search framework to identify a single optimal heuristic, optimizing average performance. \textbf{b) Automated Heuristic Set Design (AHSD)} seeks to generate a set of complementary heuristics, enhancing performance across diverse problem instances. EoH-S adopts an evolutionary framework with diversity-aware memetic search and complementary population management for effective AHSD.
}
    \label{fig:framework}
\end{figure*}

\section{Evolution of Heuristic Set (EoH-S)}

\subsection{Framework Overview}
We propose an evolutionary search framework, named Evolution of Heuristic Set (EoH-S), aimed at automatically designing a set of complementary heuristics. As illustrated in Figure~\ref{fig:framework}, the existing LLM-driven AHD methods aim at generating a single optimal heuristic to optimize the average performance on the target task~\cite{liu2024evolution,romera2024mathematical,ye2024reevo,zheng2025monte}. In contrast, EoH-S is proposed for AHSD to design a set of heuristics that complement each other on diverse instances. 

Like prior work~\cite{liu2024evolution,yao2025multi}, each heuristic in EoH-S is represented by both a high-level thought description and an executable code implementation (in this paper, Python functions). However, unlike existing approaches that use average performance as fitness, EoH-S maintains an instance-wise performance vector (i.e., scores across $m$ instances) for each heuristic, enabling complementary search and management. 

The EoH-S framework operates as follows:

\begin{enumerate}
    \item \textbf{Initialization:} EoH-S begins with an initial population, $P_0$, consisting of $n$ heuristics $\{h_1, \dots, h_n\}$. These heuristics are generated by repeatedly prompting LLMs using an initialization prompt $h_i = \text{Init}(LLM, p_i)$. The initialization prompt consists of a task description and a function template. The LLM is instructed first to generate a heuristic thought and then implement its code based on the template. Due to space limitations, all prompt details used in EoH-S are in Appendix.
    
    \item \textbf{Evolutionary Cycle:} Start from the initial population, EoH-S iteratively performs the following steps:
    \begin{enumerate}
        \item \textbf{Memetic Search:} Two strategies are adopted to create $n$ new heuristics $\{h_{o_1}, \dots, h_{o_n}\}$ based on existing ones in the current population:
        \begin{enumerate}
 \item Complementary-aware search (CS): $h_o = \text{CS}(LLM, p_{cs}, h_p)$, where $p_{cs}$ is the prompt used for complementary-aware search and $h_p$ are the parent heuristics used in the prompt. In this paper, we use two parent heuristics, which are selected according to their complementary behaviour on diverse instances. CS is used to explore diverse heuristics to enhance the complementary performance of the entire heuristic set.
 \item Local search (LS): $h_o = \text{LS}(LLM, p_{ls}, h_p)$, where $p_{ls}$ is the local search prompt and $h_p$ is one selected heuristic. Local search strategy aims at revising one heuristic to search for a new heuristic close to the parent heuristic.
        \end{enumerate}
        \item \textbf{Population Management:}
        The $n$ new heuristics and $n$ heuristics from the current population are combined into a candidate pool of size $2n$: $\hat{P}_{i+1} = P_i \cup \{h_{o_1}, \dots, h_{o_n}\}$. Then, a Complementary Population Management (CPM) strategy is used to select $n$ heuristics from these $2n$ condidates to form the next population: $P_{i+1}=CPM(\hat{P}_{i+1})$
    \end{enumerate}
    
    \item \textbf{Termination:} The process is terminated when a predefined stopping criterion is met, such as reaching the maximum number of evaluated heuristics $N_{max}$.
\end{enumerate}

\subsection{Memetic Search}
EoH-S adopts a complementary-aware memetic search to effectively explore new heuristics. It integrates two reproduction strategies: \textbf{Complementary-aware Search (CS)} and \textbf{Local Search (LS)}. %Each strategy is guided by a Large Language Model (LLM), which accepts both the parent heuristic and environment-specific prompts.

\paragraph{Complementary-aware Search (CS)}
CS encourages exploring complementary heuristics through both 1) selecting two parent heuristics $h_p$ and 2) the prompt $p_{cs}$. Specifically, two parent heuristics $h_p$ are selected based on their instance-wise performance across $m$ instances. The complementary nature between heuristics is measured using Manhattan distance $\text{d}_{h_a,h_b}$ between two heuristics $h_a$ and $h_b$. The Manhattan distance measures the sum of the absolute differences between corresponding elements, quantifying the total performance difference between two heuristics across the same set of problem instances. It is computed as:

$$
\text{d}_{h_a,h_b} = \sum_{j=1}^{m} |f_{i_j}(h_a) - f_{i_j}(h_b)|,
$$
where $F(h_a) = \{f_{i_1}(h_a), \dots, f_{i_m}(h_a)\}$ and $F(h_b) = \{f_{i_1}(h_b), \dots, f_{i_m}(h_b)\}$ denote the performance scores of heuristics $a$ and $b$ on $m$ instances. We select the pair of heuristics in the current population $P$ with the largest distance:

$$
(h_{p1}, h_{p2}) = argmax_{h_a, h_b \in P} \text{d}_{ab}.
$$

In the prompt $p_{cs}$, we inform LLM that we have this complementary pair of heuristics and instruct LLM to create new heuristics distinct from existing designs.

\paragraph{Local Search (LS)}
In contrast to CS, LS focuses on refinement of existing heuristics instead of exploring new ones. The one parent heuristic $h_p$ is selected according to a weighted random selection where better-ranked functions have higher probabilities~\cite{liu2024evolution}. The rank is sorted using the average performance on $m$ instances. This strategy promotes exploitation by producing heuristics that preserve the parent's strengths while improving performance through fine-tuned adjustments.

We generate $n$ new heuristics in memetic search using the two reproduction strategies with equal probability. Following reproduction, candidate heuristics are evaluated on a set of diverse instances. 

\subsection{Complementary Population Management}

EoH-S employs a \textbf{Complementary Population Management (CPM)} mechanism to select the $n$ heuristics from $2n$ candidate pool to form the next population. Unlike conventional LLM-driven AHD methods that select heuristics based solely on average performance, CPM leverages instance-wise performance vectors to maintain population complementarity across problem instances. It greedily select $n$ heuristics from $2n$ candidate heuristics to minimize the $\mathcal{F}(H)$ on $m$ instances, which ensures the performance guarantee established in \textbf{Theorem 3}.

Specifically, the CPM is performed as follows: Firstly, select the heuristic with average performance on $m$ instances from $2n$ heuristics to be the first one. Then, iteratively select the next heuristic to maximize the delta CPI until $n$ heuristics have been selected. Given a current heuristic set $H_k = \{h_1, \dots, h_k\}$ and a candidate heuristic $h_i$, the delta CPI is defined as:
\begin{equation}
    \Delta \text{CPI}(h_i \mid H_k) = \sum_{j=1}^m \max\left(f^*_{H_k,i_j} - f_{i_j}(h_i), 0\right),
\end{equation}
where:
\begin{itemize}
    \item $f_{i_j}(h_i)$ is the performance score of $h_i$ on instance $i_j$ (\textit{lower is better}).
    \item $f^*_{H_k,i_j} = \min_{h \in H_k} f_{i_j}(h)$ is the best score in $H_k$ for instance $i_j$.
\end{itemize}

The CPM operates as follows:
\begin{enumerate}
    \item \textbf{Initialize}: Select the heuristic with best average performance from $2n$ candidates as $h_1$.
    
    \item \textbf{Iterative Selection}:
    \begin{enumerate}
        \item For current set $H_k = \{h_1,\dots,h_k\}$, compute $\Delta \text{CPI}(h_i \mid H_k)$ for each remaining heuristic $h_i$ in $2n-k$ candidates.
        \item Select $h_{i^*} = argmax_{i \in 2n-k} \Delta \text{CPI}(h_i \mid H_k)$ and update $H_{k+1} = H_k \cup \{h_{i^*}\}$.
        \item Update reference scores: $f^*_{H_{k+1},i_j} = \min(f^*_{H_k,i_j}, f_{i_j}(h_{i^*}))$.
    \end{enumerate}
    
    \item \textbf{Terminate} when $|H_k| = n$.
\end{enumerate}

% \begin{enumerate}  
%     \item Suppose we have selected $k$ heuristics $H_k=\{h_1,\dots,h_k\}$, calculate the complementary improvement $CI_i$ of the $i$-th heuristic in the rest $2n-k$ heuristics w.r.t current $k$ heuristics on the $m$ instances. The complementary improvement of $i$-th heuristic is calculated as: $CI_i = \sum_j^m (min(f_{i_j}(h_i) - f^*_{H_k,i_j},0))$, where $f^*_{H_k,i_j}=\min_{h\in H_k} f_{i_j}(h)$, and $f_{i_j}(h)$ denotes the performance score of heuristic $h$ on instance $i_j$ (lower is better). 
%     \item Select the next heuristic with the largest complementary improvement on $m$ instances: $argmin_{i\in 2n-k}(CI_i)$. Add the selected heuristic into the current population $H_k$ = $H_k \cup h_{argmin_{i\in 2n-k}(CI_i)}$ and update the index $k=k+1$.
%     \item Update the performance score of the current heuristic set $H_k$ on $m$ instances: $f^*_{H_k,i_j}=\min_{h\in H_k} f_{i_j}(h)$.
% \end{enumerate}

\section{Experimental Studies}

\begin{table*}[th]

\resizebox{\textwidth}{!}{\begin{tabular}{llllllllll}
\toprule
      & \multicolumn{3}{c}{Training (c100)}         & \multicolumn{6}{c}{Testing (c200-500)}  \\
\multirow{-2}{*}{Methods} & n200-500  & n500-1k & n1k-2k& n1k\_c200 & n1k\_c500 & n5k\_c200 & n5k\_c500 & n10k\_c200& n10k\_c500\\
\midrule
First Fit       & 0.0652    & 0.0318    & 0.0387    & 0.0240    & \cellcolor[HTML]{D0D0D0}\textbf{0.0124} & 0.0173    & 0.0075    & 0.0163    & 0.0055    \\
Best Fit        & 0.0627    & 0.0310    & 0.0372    & 0.0220    & \cellcolor[HTML]{D0D0D0}\textbf{0.0124} & 0.0161    & 0.0070    & 0.0154    & 0.0050    \\
\midrule
Random*& 0.3638    & 0.2988    & 0.1489    & 0.1939    & 1.3508    & 0.0376    & 0.2816    & 0.0189    & 0.1402    \\
1+1 EPS*   & 0.2482    & 0.2180    & 0.1246    & 0.0569    & 0.2390    & 0.0101    & 0.0439    & 0.0055    & 0.0224    \\
FunSearch*       & 0.2510    & 0.2000    & 0.1326    & 0.0827    & 0.6169    & 0.0193    & 0.2260    & 0.0120    & 0.1211    \\
EoH*   & 0.1216    & 0.1169    & 0.0750    & 0.0190    & \cellcolor[HTML]{D0D0D0}\textbf{0.0124} & \cellcolor[HTML]{E9E9E9}0.0044& 0.0050    & 0.0041    & 0.0042    \\
MEoH*  & 0.0991    & 0.0749    & 0.0547    & 0.1037    & 0.0149    & 0.0417    & 0.0035    & 0.0326    & \cellcolor[HTML]{E9E9E9}0.0015\\
CALM*  & 0.0805    & 0.0300    & 0.0470    & \cellcolor[HTML]{E9E9E9}0.0120& \cellcolor[HTML]{D0D0D0}\textbf{0.0124} & 0.0048    & \cellcolor[HTML]{E9E9E9}0.0030& \cellcolor[HTML]{E9E9E9}0.0034& 0.0025    \\
ReEvo* & 0.0877    & 0.0344    & 0.0556    & \cellcolor[HTML]{E9E9E9}0.0120& \cellcolor[HTML]{D0D0D0}\textbf{0.0124} & 0.0074    & 0.0040    & 0.0068    & 0.0025    \\
MCTS-AHD*  & 0.0974    & 0.0360    & 0.0596    & 0.0150    & \cellcolor[HTML]{D0D0D0}\textbf{0.0124} & 0.0046    & 0.0040    & \cellcolor[HTML]{E9E9E9}0.0034& \cellcolor[HTML]{FFFFFF}0.0022\\
\midrule
FunSearch       & 0.0623    & 0.0298    & 0.0367    & 0.0213    & \cellcolor[HTML]{D0D0D0}\textbf{0.0124} & 0.0159    & 0.0070    & 0.0152    & 0.0051    \\
FunSearch Top10 & 0.0608    & 0.0289    & 0.0361    & 0.0200    & \cellcolor[HTML]{D0D0D0}\textbf{0.0124} & 0.0153    & 0.0065    & 0.0123    & 0.0038    \\
EoH   & 0.0619    & 0.0307    & 0.0371    & 0.0240    & 0.0133    & 0.0182    & 0.0075    & 0.0170    & 0.0059    \\
EoH Top10       & 0.0618    & 0.0299    & 0.0368    & 0.0216    & \cellcolor[HTML]{D0D0D0}\textbf{0.0124} & 0.0158    & 0.0065    & 0.0152    & 0.0051    \\
ReEvo & 0.0604    & 0.0298    & 0.0362    & 0.0217    & 0.0133    & 0.0162    & 0.0068    & 0.0155    & 0.0053    \\
ReEvo Top10     & \cellcolor[HTML]{E9E9E9}0.0580& \cellcolor[HTML]{E9E9E9}0.0285& \cellcolor[HTML]{E9E9E9}0.0349& 0.0203    & \cellcolor[HTML]{D0D0D0}\textbf{0.0124} & 0.0125    & 0.0063    & 0.0108    & 0.0049    \\
EoH-S & \cellcolor[HTML]{D0D0D0}\textbf{0.0515} & \cellcolor[HTML]{D0D0D0}\textbf{0.0230} & \cellcolor[HTML]{D0D0D0}\textbf{0.0314} & \cellcolor[HTML]{D0D0D0}\textbf{0.0113} & \cellcolor[HTML]{D0D0D0}\textbf{0.0124} & \cellcolor[HTML]{D0D0D0}\textbf{0.0033} & \cellcolor[HTML]{D0D0D0}\textbf{0.0025} & \cellcolor[HTML]{D0D0D0}\textbf{0.0014} & \cellcolor[HTML]{D0D0D0}\textbf{0.0010} \\
\bottomrule
\end{tabular}}
\caption{Evaluation of heuristics designed by different methods on OBP instances. We use 128 Weibull instances with sizes ranging from 200 to 2k for training and six different sets (n1k\_c200, n1k\_c500, n5k\_c200, n5k\_c500, n10k\_c200, n10k\_c500) for testing, where n and c represent the number of items and the bin capacity, respectively. Each method is run three times and we report the average performance. We directly use the best-performing heuristics from their original papers for the methods with $*$. The best values are \textbf{in bold} with a grey background and the second-best values are in a light-grey background. }~\label{table:main_obp}

\end{table*}

\subsection{Tasks and Instances}
We investigate the following three tasks, with detailed descriptions provided in the Appendix:
\begin{itemize}
    \item Online Bin Packing (OBP) involves packing items into the minimum number of fixed-capacity bins as items arrive sequentially. The heuristic is used to select the assigned bin for each incoming item. The performance is measured as the relative gap to the lower bound of the optimal number of bins computed as in~\cite{martello1990lower}. For heuristic evaluation during automated design (training), $I$ consists of 128 Weibull instances~\cite{romera2024mathematical} with a bin capacity of 100 and the number of items ranging from 200 to 2000 (2k). For testing,  we use six sets with larger capacities {200,500} and more items {1k, 5k, 10k}, each set consists of 5 instances following existing works~\cite{romera2024mathematical,liu2024evolution}
    \item  Traveling Salesman Problem (TSP) requires finding the shortest route visiting all cities exactly once before returning to the start. We design a step-by-step construction heuristic. The heuristic is used to iteratively select the next city to visit. The average relative gap from the baseline solution generated by LKH~\cite{helsgaun2017extension} is used for performance measurement. For training, $I$ consists of 128 instances with the number of cities ranging from 10 to 200, sampled in $[0,1]$ using a Gaussian distribution in clusters following~\citet{bi2022learning}. For testing, we use instances in uniform distribution with the number of cities ranging from 50 to 500.
    \item Capacitated Vehicle Routing Problem (CVRP) extends TSP by incorporating vehicle capacity constraints and customer demands. The goal is minimizing total travel distance while ensuring all customer demands are met without exceeding vehicle capacities. We design a step-by-step construction heuristic. The heuristic is to select the next node. The average relative gap from the baseline solution generated by LKH~\cite{helsgaun2017extension} is used for performance measurement. For training, $I$ consists of 256 instances with the number of nodes from 20 to 200 and the capacities from 10 to 150. For testing, we use instances with node sizes ranging from 50 to 500. For both training and testing, the nodes are sampled in $[0,1]$ using a uniform distribution. We do not consider Gaussion distribution here because the different capacities and sizes have already introduced diversities.
\end{itemize}

\subsection{Compared Methods and Settings}

We compare state-of-the-art LLM-driven AHD methods, including \textbf{Random} (i.e., repeated prompt LLMs to generate heuristic without an iterative search framework~\cite{zhang2024understanding}),  \textbf{EoH}~\cite{liu2023algorithm,liu2024evolution},  \textbf{FunSearch}~\cite{romera2024mathematical}, \textbf{ReEvo}~\cite{ye2024reevo}, \textbf{1+1 EPS}~\cite{zhang2024understanding}, \textbf{MEoH}~\cite{yao2025multi} \textbf{MCTS-AHD}~\cite{zheng2025monte}, and \textbf{CALM}~\cite{huang2025calm}.

We evaluate these methods under two settings: \textbf{1) Direct comparison}: We directly compare the best-performing heuristics reported in the original papers for the online bin packing problem. These heuristics are all designed for the Weibull instances with slightly different distributions. \textbf{2) Controlled comparison}: For some representative methods (EoH, FunSearch, and ReEvo), we conduct training using identical instances to EoH-S, enabling a fair performance comparison. Three independent runs are performed for these methods on all three tasks. 

All experiments are conducted on LLM4AD platform~\cite{liu2024llm4ad} using \textsc{DeepSeek-V3}~\cite{liu2024deepseek} API with default parameter settings. We set the maximum number of heuristic evaluations to $N_{\text{max}} = 2,\!000$ for all tasks and fix the population size at $n = 10$ for EoH-S, EoH, and ReEvo to maintain consistency (FunSearch uses a dynamic population size). The experiments run on one Intel i7-9700 CPU, and the automated heuristic design process completes in under two hours for all methods across different tasks. It is worth noting that although EoH-S explicitly design a set of complementary heuristics, it does not introduce additional running time because we use the same maximum number of evaluations for all methods.

%All experiments are conducted use DeepSeek-V3~\cite{liu2024deepseek} API with default parameter settings. We set the maximum number of heuristic evaluations to $N_{max}=2,000$ across all tasks. Population sizes are fixed at $n=10$ for EoH-S, EoH, and ReEvo to ensure consistent comparison. The experiments run on CPU i7-9700, and the automated heuristic design process takes less than 2 hours for all methods on different tasks.

\subsection{Results on Training \& Testing Instances}
The results on training and testing instances of diverse distributions and sizes are summarized in Table~\ref{table:main_obp} (OBP) and Table~\ref{table:main_tsp_cvrp} (TSP and CVRP), where the best values are in bold with a grey background and the second-best values are in a light-grey background. The values are the gap to baseline results (lower bound for OBP~\cite{romera2024mathematical} and LKH for TSP and CVRP~\cite{helsgaun2017extension}). We report the average results over three independent runs. The methods with $*$ denote the best one heuristic from their original paper. For EoH-S, we use 10 heuristics in the final population. For fair comparison, we also compare the top 10 heuristics from FunSearch, EoH, and ReEvo (denoted as FunSearch Top10, EoH Top10, and ReEvo Top10, respectively). 

As indicated in Table~\ref{table:main_obp}, EoH-S consistently outperforms all baseline methods on both training instances (c100 with varying item counts) and testing instances (c200-500 with different problem sizes), achieving the lowest gap values in 8 out of 9 configurations and matching the best performance in the remaining one. In contrast, existing LLM-driven AHD methods from the literature struggle to generalize across different distributions, with some even underperforming first-fit and best-fit heuristics, which is also observed by~\citet{xu2010hydra}.

Table~\ref{table:main_tsp_cvrp} shows that EoH-S performs consistently well on TSP and CVRP. On TSP, EoH-S significantly reduces the optimality gap by 50-60\% compared to the second-best method across all test instances of varying sizes (50-500 nodes). For CVRP, EoH-S again achieves the lowest optimality gaps on both training and testing instances, with particularly significant improvements on smaller problem instances. 

%These consistent results across different problem domains and instance sizes demonstrate the effectiveness and robustness of our proposed approach compared to existing methods, including classical heuristics and learning-based alternatives.

\begin{table}[t] 
\resizebox{\linewidth}{!}{
\begin{tabular}{llllll}
\toprule
\multicolumn{1}{c}{}           & \multicolumn{1}{c}{}& \multicolumn{4}{c}{Testing}\\
\multicolumn{1}{c}{\multirow{-2}{*}{Methods}} & \multicolumn{1}{c}{\multirow{-2}{*}{Training}} & 50       & 100      & 200      & 500      \\
\midrule
FunSearch       & 0.156& 0.144    & 0.149    & 0.169    & 0.185    \\
FunSearch Top10 & 0.124& 0.091    & \cellcolor[HTML]{E9E9E9}0.105          & \cellcolor[HTML]{E9E9E9}0.124          & \cellcolor[HTML]{E9E9E9}0.157          \\
EoH & 0.152& 0.142    & 0.154    & 0.168    & 0.194    \\
EoH Top10      & 0.127& \cellcolor[HTML]{E9E9E9}0.079          & 0.123    & 0.134    & 0.175    \\
ReEvo           & 0.155& 0.167    & 0.183    & 0.223    & 0.221    \\
ReEvo Top10    & \cellcolor[HTML]{E9E9E9}0.117   & 0.092    & 0.123    & 0.157    & 0.170    \\
EoH-S           & \cellcolor[HTML]{D0D0D0}\textbf{0.097}         & \cellcolor[HTML]{D0D0D0}\textbf{0.040} & \cellcolor[HTML]{D0D0D0}\textbf{0.065} & \cellcolor[HTML]{D0D0D0}\textbf{0.090} & \cellcolor[HTML]{D0D0D0}\textbf{0.111} \\
\midrule
\midrule
\multicolumn{1}{c}{}           & \multicolumn{1}{c}{}& \multicolumn{4}{c}{Testing}\\
\multicolumn{1}{c}{\multirow{-2}{*}{Methods}} & \multicolumn{1}{c}{\multirow{-2}{*}{Training}} & 50       & 100      & 200      & 500      \\
\midrule
FunSearch       & 0.294& 0.315    & 0.365    & 0.296    & 0.241    \\
FunSearch Top10 & 0.245& 0.267    & 0.309    & 0.247    & 0.214    \\
EoH & 0.276& 0.274    & 0.299    & 0.261    & 0.221    \\
EoH Top10      & \cellcolor[HTML]{E9E9E9}0.230& \cellcolor[HTML]{E9E9E9}0.172          & \cellcolor[HTML]{E9E9E9}0.232          & 0.218    & 0.198    \\
ReEvo           & 0.265& 0.274    & 0.296    & 0.256    & 0.203    \\
ReEvo Top10    & 0.240   & 0.214    & 0.249    & \cellcolor[HTML]{E9E9E9}0.173          & \cellcolor[HTML]{E9E9E9}0.178          \\
EoH-S           & \cellcolor[HTML]{D0D0D0}\textbf{0.173}         & \cellcolor[HTML]{D0D0D0}\textbf{0.135} & \cellcolor[HTML]{D0D0D0}\textbf{0.188} & \cellcolor[HTML]{D0D0D0}\textbf{0.180} & \cellcolor[HTML]{D0D0D0}\textbf{0.169} \\
\bottomrule 
\end{tabular}
}
\caption{Evaluation of heuristics designed by different methods on TSP (upper) and CVRP (lower). We use 128 TSP instances with sizes ranging from 10 to 200 and 256 CVRP instances with sizes ranging from 20 to 200 for training, and four different sets with sizes ranging from 50 to 500 for testing. Each method is run three times, and we report the average performance. The best values are \textbf{in bold} with a grey background and the second-best values are in a light-grey background. }~\label{table:main_tsp_cvrp}

\end{table}

\begin{table}[t]
\resizebox{\linewidth}{!}{
\begin{tabular}{llllll}
\toprule
\multicolumn{1}{c}{}     & \multicolumn{2}{c}{EoH}   & \multicolumn{2}{c}{ReEvo} & \multicolumn{1}{c}{}\\
\multicolumn{1}{c}{\multirow{-2}{*}{Benchmarks}} & Top 1 & Top 10& Top 1 & Top 10& \multicolumn{1}{c}{\multirow{-2}{*}{EoH-S}} \\
\midrule
BPPLib Sch\_1       & \cellcolor[HTML]{E9E9E9}0.153 & \cellcolor[HTML]{E9E9E9}0.153 & \cellcolor[HTML]{E9E9E9}0.153 & \cellcolor[HTML]{E9E9E9}0.153 & \cellcolor[HTML]{D0D0D0}\textbf{0.152}      \\
BPPLib Sch\_2       & \cellcolor[HTML]{E9E9E9}0.142 & \cellcolor[HTML]{E9E9E9}0.142 & \cellcolor[HTML]{E9E9E9}0.142 & \cellcolor[HTML]{E9E9E9}0.142 & \cellcolor[HTML]{D0D0D0}\textbf{0.141}      \\
BPPLib IRUP     & 0.077 & 0.075 & 0.079 & \cellcolor[HTML]{E9E9E9}0.074 & \cellcolor[HTML]{D0D0D0}\textbf{0.072}      \\
BPPLib NonIRUP  & 0.077 & 0.076 & 0.080 & \cellcolor[HTML]{E9E9E9}0.075 & \cellcolor[HTML]{D0D0D0}\textbf{0.073}      \\
BPPLib Scholl\_H      & \cellcolor[HTML]{E9E9E9}0.138 & \cellcolor[HTML]{E9E9E9}0.138 & \cellcolor[HTML]{E9E9E9}0.138 & \cellcolor[HTML]{E9E9E9}0.138 & \cellcolor[HTML]{D0D0D0}\textbf{0.095}      \\
\midrule
TSPLib       & 0.184 & 0.173 & 0.220 & \cellcolor[HTML]{E9E9E9}0.165 & \cellcolor[HTML]{D0D0D0}\textbf{0.093}      \\
\midrule
CVRPLib A    & 0.326 & 0.277 & 0.329 & \cellcolor[HTML]{E9E9E9}0.261 & \cellcolor[HTML]{D0D0D0}\textbf{0.231}      \\
CVRPLib B    & 0.374 & 0.310 & 0.352 & \cellcolor[HTML]{E9E9E9}0.250 & \cellcolor[HTML]{D0D0D0}\textbf{0.183}      \\
CVRPLib E    & 0.334 & 0.268 & 0.305 & \cellcolor[HTML]{E9E9E9}0.257 & \cellcolor[HTML]{D0D0D0}\textbf{0.242}      \\
CVRPLib F    & 0.539 & \cellcolor[HTML]{E9E9E9}0.493 & 0.687 & 0.547 & \cellcolor[HTML]{D0D0D0}\textbf{0.427}      \\
CVRPLib M    & 0.461 & 0.377 & 0.442 & \cellcolor[HTML]{E9E9E9}0.372 & \cellcolor[HTML]{D0D0D0}\textbf{0.299}      \\
CVRPLib P    & 0.273 & 0.206 & 0.270 & \cellcolor[HTML]{E9E9E9}0.188 & \cellcolor[HTML]{D0D0D0}\textbf{0.169}      \\
CVRPLib X    & 0.270 & 0.237 & 0.270 & \cellcolor[HTML]{E9E9E9}0.224 & \cellcolor[HTML]{D0D0D0}\textbf{0.196}     \\
\bottomrule
\end{tabular}}
\caption{Results on OPPLib, TSPLib, and CVRPLib Benchmarks. The best values are \textbf{in bold} with a grey background and the second-best values are in a light-grey background.}~\label{table:benchmark}
\end{table}

\subsection{Results on Benchmark Instances}

Tabel~\ref{table:benchmark} presents a comprehensive evaluation of our proposed EoH-S method against existing approaches (EoH and ReEvo) across multiple benchmark datasets, including BPPLib~\cite{delorme2018bpplib}, TSPLib~\cite{reinelt1991tsplib}, and CVRPLib~\cite{uchoa2017new}. Our proposed EoH-S method consistently outperforms both baseline methods across all benchmark instances. On the BPPLib datasets, EoH-S achieves superior results, with particularly notable improvement on the Scholl\_H instance (0.095 compared to 0.138 for both baselines). 

For the TSPLib benchmark, EoH-S demonstrates substantial performance gains with a score of 0.093, representing a 43.6\% improvement over the second-best result (0.165 from ReEvo Top 10). On the CVRPLib benchmarks, EoH-S achieves improvements ranging from 5.8\% (on set E) to 26.8\% (on set B) compared to the second-best results. These results clearly demonstrate the effectiveness and robustness of our proposed approach in designing complementary heuristics.

\begin{figure}[t]  % H option forces exact placement
    \centering
    \begin{subfigure}[b]{0.9\linewidth}
        \centering
        \includegraphics[width=\linewidth]{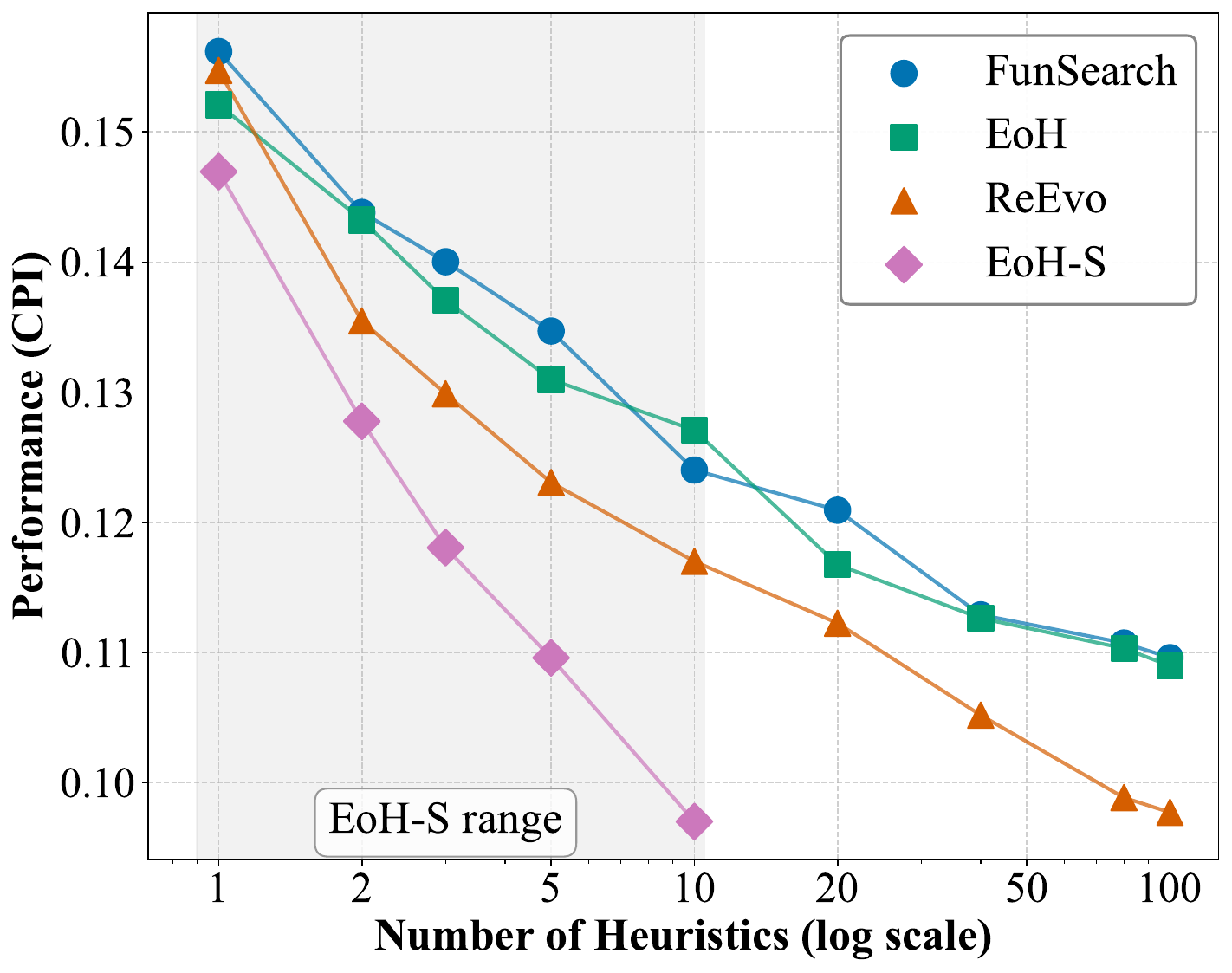}
        \label{fig:comp_tsp}
    \end{subfigure}
    
    %\vspace{-0.5em}  % Reduced from 1em for tighter spacing
    
    \begin{subfigure}[b]{0.9\linewidth}
        \centering
        \includegraphics[width=\linewidth]{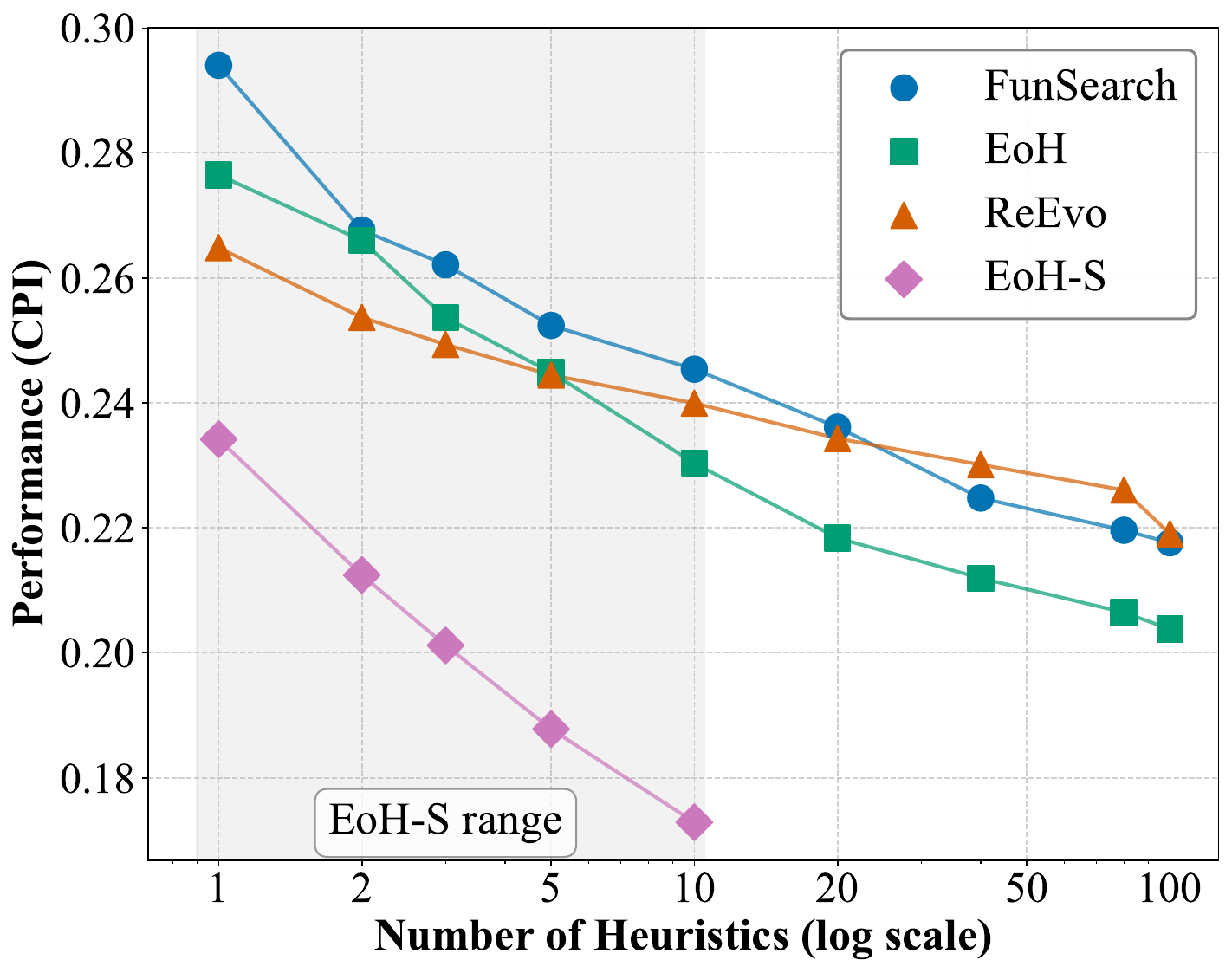}
        \label{fig:comp_cvrp}
    \end{subfigure}
    \caption{Complementary Performance Index (CPI) comparison across methods with varying numbers of heuristics. For EoH-S, we select heuristics from the final population (maximum 10), while for other methods, we select the best 100 heuristics from the entire search history.}
    \label{fig:comparison}
\end{figure}

\subsection{Complementary Performance}

We evaluate the complementary performance of heuristic sets generated by different methods. Figure~\ref{fig:comparison} compares the performance (averaged over three independent runs) of heuristic sets as the number of heuristics increases (on a logarithmic scale from 1 to 100). The CPI (Complementary Performance Index) measures the gap to the lower bound.

For EoH-S, we use the final population of heuristics, capped at 10. For EoH, FunSearch, and ReEvo, we select the top 100 heuristics from all candidates evaluated during the search process. When the set size is 1, we report the performance of the single best heuristic (in terms of average performance). For a set size of 10, we consider the full heuristic set for EoH-S and the top 10 heuristics for other methods. The CPI drop observed when expanding from a single heuristic to a larger set reflects the degree of complementarity, i.e., how effectively additional heuristics contribute to solving diverse problem instances.

% We examine the complementary performance of heuristics designed by different methods. Figure~\ref{fig:comparison} illustrates the performance (averaged over three independent runs) of heuristic sets relative to the number of heuristics (shown on a logarithmic scale ranging from 1 to 100). The CPI is the gap to the lower bound. For EoH-S, we utilize the heuristics from the final population, with a maximum of 10 heuristics. For EoH, FunSearch, and ReEvo, we select the top 100 heuristics from all candidates evaluated during the automated heuristic search process. When the number of heuristic is $1$, it presents using the single-best heuristic with best average performance. When the number of heuristics is $10$ we use the entire heuristic set for EoH and the best top 10 heuristics for other methods. The drop in the performance from heuristic set of size 1 to a larger size represents the degree of complementary by including additional heuristics. 

%The heuristics in the final population are included in this set of 100 due to the greedy population management approach used by these methods.

Our results demonstrate three key findings: First, the heuristic set designed by EoH-S consistently outperforms those designed by the compared methods. Second, even when utilizing 100 heuristics, the competing methods struggle to match the performance achieved by just 10 heuristics designed by EoH-S. Third, each heuristic in the EoH-S set makes a meaningful contribution to the overall performance. The significantly steeper performance curve of EoH-S compared to other methods indicates that its heuristics provide greater complementary benefits.

\subsection{Ablation Studies}
We conduct ablation studies on OBP to evaluate the contribution of each key component and the impact of population size. The following variants of EoH-S are examined:

\begin{itemize}
\item EoH-S w/o LS: Excludes the Local Search operator (LS).
\item EoH-S w/o CS: Excludes the Complementary-aware Search operator (CS).
\item EoH-S w/o CPM: Disables Complementary Population Management (CPM) and instead uses a population management approach in existing works (retaining the best-performing heuristics based on average performance).
\item EoH-S n5–n20: Tests EoH-S with population sizes ranging from 5 to 20.
\end{itemize}

The experimental setup matches the main experiments, with 2,000 total samples per run and a population size 10.

Table~\ref{tab:ablation} presents the average gap to the lower bound across 128 OBP training instances. The first three methods are two hand-designed heuristics and the ReEvo Top 10 (the second-best one, as indicated in Table~\ref{table:main_obp}). The No. 4 to No. 6 are three variants of EoH-S that assess the importance of the search operators and CPM. Results show that all components contribute to performance, with CPM providing the most significant improvements. The remaining variants (No. 7–9) demonstrate that EoH-S consistently outperforms existing methods across population sizes (5–20), with a size of 10 proving optimal for this task. Results demonstrate the contribution of different components and the robustness of EoH-S across different settings.

\begin{table}[tbp]
\centering
\small
\begin{tabular}{lll}
\toprule
No. & Settings      & Gap    \\
\midrule
1   & First Fit     & 0.0413 \\
2   & Best Fit      & 0.0399 \\
3   & ReEvo Top 10  & 0.0371 \\
\midrule
4   & EoH-S w/o LS  & 0.0336 \\
5   & EoH-S w/o CS  & 0.0335 \\
6   & EoH-S w/o CPM & 0.0373 \\
\midrule
7   & EoH-S n5      & 0.0333 \\
8   & EoH-S n15     & 0.0339 \\
9   & EoH-S n20     & 0.0337 \\
\midrule
10  & EoH-S         & 0.0326 \\
\bottomrule
\end{tabular}
\caption{Ablation study of key components and settings}~\label{tab:ablation}
\end{table}

\section{Conclusion and Future Works}

In this work, we introduced Automated Heuristic Set Design (AHSD), a novel formulation for LLM-driven AHD that generates a small yet complementary set of heuristics to optimize diverse problem instances effectively. We demonstrated that the AHSD objective function is monotone and supermodular, providing a theoretical foundation for efficient heuristic set optimization. We proposed Evolution of Heuristic Set (EoH-S), which integrates a complementary population management and a memetic search for effectively evolving high-quality and complementary heuristics. Extensive experiments across multiple AHD tasks validated the superiority of EoH-S, achieving up to 60\% performance improvements over state-of-the-art AHD methods. Notably, EoH-S demonstrated strong generalization capabilities using a small set of heuristics.

Our work advances LLM-driven heuristic design by shifting from single-best heuristic design to complementary heuristic set design. Future research directions include exploring heuristic collaboration strategies for further performance gains and additional application studies.

% In this work, we proposed Evolution of Heuristic Set (EoH-S), a novel framework for LLM-driven automated heuristic design (AHD) that addresses the limitations of single-optimal-heuristic approaches. By evolving a complementary set of heuristics with complementary population management strategy and diversity-aware memetic search on diverse problem instances, EoH-S significantly improves generalization across varying distributions and scales. Our experiments on three optimization tasks demonstrate that EoH-S outperforms existing methods, achieving up to 60\% average gap reduction on synthetic and real-world benchmarks. EoH-S opens new avenues for heuristic set design for robust, generalizable automated algorithm generation.

\bibliography{eohs}

\clearpage
\newpage
\appendix
\onecolumn
% \addcontentsline{toc}{section}{Appendix}

% \renewcommand{\thesubsection}{\Alph{subsection}}
% \setcounter{subsection}{0} % Reset counter for subsections

% % Table of Contents for Appendix
% \subsection*{Appendix: Table of Contents}
% \begin{itemize}
%     \item \hyperref[sec:related_works]{\textbf{A. Related Works}}
%     \begin{itemize}
%         \item \hyperref[subsec:ahd]{A.1. Automated Heuristic Design (AHD)}
%         \item \hyperref[subsec:nco]{A.2. Neural Combinatorial Optimization}
%         \item \hyperref[subsec:llm_ahd]{A.3. LLM-driven AHD}
%     \end{itemize}
    
%     \item \hyperref[sec:heuristic_tasks]{\textbf{B. Heuristic Design Tasks  \& Setups}}
%     \begin{itemize}
%         \item \hyperref[subsec:obp]{B.1. Online Bin Packing (OBP)}
        
%         \item \hyperref[subsec:tsp]{B.2. Traveling Salesman Problem (TSP)}
        
%         \item \hyperref[subsec:cvrp]{B.3. Capacitated Vehicle Routing Problem (CVRP)}

%     \end{itemize}
    
%     \item \hyperref[sec:method_details]{\textbf{C. Method Details}}
%     \begin{itemize}
%         \item \hyperref[subsec:initialization]{C.1. Initialization}
%         \item \hyperref[subsec:diversity_search]{C.2. Diversity-aware Search}
%         \item \hyperref[subsec:local_search]{C.3. Local Search}
%     \end{itemize}
    
%     \item \hyperref[sec:experimental_results]{\textbf{D. Additional Experimental Results}}
%     \begin{itemize}
%         \item \hyperref[subsec:convergence]{D.1. Convergence Analysis \& Complementary Performance Analysis}
%         \item \hyperref[sec:designed_heuristics]{D.3. Designed Heuristics}
%     \item \hyperref[sec:benchmark_results]{D.4. Benchmark Evaluation Details}
%     \begin{itemize}
%         \item \hyperref[subsec:bpplib]{D.4.1. BPPLib Results}
%         \item \hyperref[subsec:tsplib]{D.4.2. TSPLib Results}
%         \item \hyperref[subsec:cvrplib]{D.4.3. CVRPLib Results}
%     \end{itemize}
%     \end{itemize}
% \end{itemize}

\section{Related Works}~\label{sec:related_works}
\subsection{Automated Heuristic Design (AHD)}
Automated heuristic design, often referred to as hyperheuristics~\cite{burke2018classification,stutzle2018automated}, which are designed to automate the selection, combination, generation, or adaptation of simpler heuristics to efficiently solve computational search problems~\cite{pillay2018hyper}. 

%Hyperheuristics are broadly categorized into four types based on their function (selection or generation) and the nature of the low-level heuristics (constructive or perturbative). Selection constructive methods choose appropriate heuristics during solution construction, employing techniques like case-based reasoning~\cite{burke2006case} and hybrid approaches~\cite{qu2009hybridizations}, while selection perturbative methods modify existing solutions using single-point or multipoint search strategies~\cite{burke2013hyper}.

In automated heuristic generation, generation constructive hyperheuristics synthesize new heuristics, often via genetic programming or grammatical evolution~\cite{o2001grammatical}, producing either instance-specific or reusable heuristics. Generation perturbative methods combine existing perturbative heuristics with acceptance criteria. Recent work extends these classifications through component-based unified solvers, integrating diverse operators and stages to form new algorithms~\cite{bezerra2015automatic,qu2020general}, which leverage fundamental components to address a wide range of optimization problems. 

Although methods like genetic programming~\cite{langdon2013foundations} and recent unified solvers~\cite{bezerra2015automatic,qu2020general} offer a flexible and explainable approach to algorithm design, they still necessitates tailored hand-crafted components and much domain-specific knowledge.

%leverages a range of effective methodologies and frameworks to automatically tune heuristics or integrate various algorithmic components. Significant advancements have been made by incorporating machine learning techniques into automated heuristic design~\cite{bengio2021machine}. 

\subsection{Neural Combinatorial Optimization}

In the past decade, neural combinatorial optimization has emerged as a promising paradigm that trains neural networks to learn heuristics for solving complex optimization problems \cite{bengio2021machine}. This approach has garnered significant attention due to its effectiveness and computational efficiency in generating high-quality solutions.

Cross-distribution learning is a key research direction within this field that aim to enhance generalization performance across diverse problem distributions and scales. For instance, \citet{jiang2022learning} and \citet{bi2022learning} investigated robust optimization techniques across multiple geometric distributions. Complementary work by several researchers \cite{fu2021generalize, pan2023h-tsp, manchanda2023generalization, luo2023neural} has explored methods to improve generalization to larger problem instances. Moreover, recent attempts~\cite{zhou2023towards,liu2024prompt} have expanded it by simultaneously addressing both problem size and geometric distribution variations.

Despite these notable advances, neural combinatorial solvers face significant challenges. They typically require substantial expertise in model architecture design and significant computational resources for training. Furthermore, the inherent black-box nature of these neural approaches limits their interpretability, which presents a considerable barrier to their adoption in real-world applications where solution transparency and explainability are often critical requirements.

% In the last ten years, neural combinatorial optimization, which trains a neural network to learn a heuristics for addressing optimization problems~\cite{bengio2021machine}, has gained attention due to its effectiveness and efficiency in solution generation. Several cross-distribution learning approaches have attempted to solve generalization performance across different problem distributions and scales. For instance, \citet{jiang2022learning} and \citet{bi2022learning} explored robust optimization over multiple geometric distributions, while others~\cite{fu2021generalize,pan2023h-tsp,manchanda2023generalization,luo2023neural} investigated generalization to larger problem instances. Zhou~\shortcite{zhou2023towards} further extended the scope by considering both problem size and geometric distribution. 

% Despite these efforts, neural solvers require extensive model crafting and training. Moreover, the black-box nature of neural solvers lacks interpretability and hinders its real-world application.

\subsection{LLM-driven AHD}
Among existing LLM-driven AHD methods, a prominent strategy involves iteratively using LLMs to design and refine algorithms within an evolutionary framework~\cite{zhang2024understanding}. This paradigm has demonstrated promising results across diverse algorithm design tasks, including combinatorial optimization, Bayesian optimization, and black-box optimization~\cite{van2024llamea}. Notable examples include EoH~\cite{liuevolution}, which evolves both thoughts and code through five distinct prompt strategies to guide effective algorithm search, significantly enhancing the exploration of the solution space. FunSearch~\cite{romera2024mathematical} employs a multi-island evolutionary framework but relies on a single prompt strategy to instruct LLMs in algorithm refinement. Moreover, ReEvo~\cite{ye2024reevo} incorporates short- and long-term reflection strategies into the evolutionary process to provide additional guidance to LLMs. In addition to evolutionary search, other search frameworks, such as Monte Carlo tree search (MCTS)~\cite{zheng2025monte} and neighborhood search~\cite{xie2025llm}, have also been investigated to enhance the efficiency. 

While these methods have achieved success, they mainly focus on finding a single heuristic that has optimal average performance, which usually suffers from poor generalization performance beyond their training distribution~\cite{sim2025beyond,shi2025generalizable}.

\clearpage
\section{Heuristic Design Tasks \& Setups}

\subsection{Online Bin Packing (OBP)}~\label{apd:tasks}

\subsubsection*{Problem Definition}  
OBP involves packing a sequence of items into a minimum number of bins, each with a fixed capacity $C \in \mathbb{R}^+$. Items arrive sequentially and must be assigned to a bin upon arrival, without knowledge of future items.

\begin{itemize}  
    \item Let $\mathcal{I} = \{i_1, i_2, \dots, i_n\}$ denote a sequence of $n$ items, where each item $i_j$ has a \textbf{size} $s_j \in (0, C]$.  
    \item A \textbf{bin} is a container with capacity $C$, and a packing is valid if the sum of item sizes in any bin does not exceed $C$:  
    $$
    \forall \text{ bin } B_k, \sum_{i_j \in B_k} s_j \leq C.
    $$  
    \item The objective is to minimize the total number of bins used:  
    $$
    K = |\{B_1, B_2, \dots, B_K\}|.
    $$  
\end{itemize}  

In the online setting: 1) Items arrive one-by-one in an unknown order. 2) Each item $i_j$ must be assigned to a bin immediately upon arrival. 3) No reallocation or rearrangement of previously packed items is allowed.

\subsubsection*{Designed Heuristic}  
The design heuristic is used to calculate the priorities of the bins for each incoming item, given the size of the item and the rest capacities of the bins. The heuristic is implemented as a Python function. The performance is measured by the average relative gap to lower bound of the number of used bins.

\subsubsection*{Data Generation}  
\begin{itemize}
\item \textbf{Training Instances:} For training, we sample 128 instances using the Weibull distribution with shape parameters $k \in \{1, 3, 5\}$ and scale parameters $\lambda \in \{5, 10, 20, 40, 80\}$. The number of items per instance ranges from 200 to 2000, randomly sampled from this interval. All training instances use a fixed capacity of 100.

\item \textbf{Testing Instances:} For testing, we create six sets with larger capacities $\{200, 500\}$ and more items $\{1\text{k}, 5\text{k}, 10\text{k}\}$. Each set contains 5 instances, following the settings established in prior work~\cite{romera2024mathematical,liu2024evolution}.

    % \item \textbf{Training Instance:} For training, we use 128 instances sampled using the Weibull distribution with the combinations of shape parameter $\{1,3,5\}$ scale parameter lambda $\{5, 10, 20, 40, 80\}$. The number of items for each instance are randomly sampled at $[200, 2000]$ and the capacity is set to 100 for all instances.
    % \item \textbf{Testing Instance: } For testing, we use six sets with larger capacities {200,500} and more items {1k, 5k, 10k}, each set consists of 5 instances using the settings following existing works~\cite{romera2024mathematical,liu2024evolution}
\end{itemize}

The task description and template used in the prompts are as follows:

\begin{tcolorbox}[colback=gray!5, colframe=gray!50]
\textbf{OBP Task Description:} Implement a function that returns the priority with which we want to add an item to each bin.
\end{tcolorbox}

\begin{figure}[ht]
\centering
\begin{mdframed}[
    linecolor=gray!50,
    linewidth=1pt,
    roundcorner=5pt,
    backgroundcolor=gray!5,
    innertopmargin=10pt,
    innerbottommargin=10pt,
    innerrightmargin=10pt,
    innerleftmargin=10pt,
    frametitlebackgroundcolor=gray!20,
    frametitleaboveskip=3pt,
    frametitlebelowskip=6pt,
    frametitlefont=\bfseries\sffamily
]
\begin{lstlisting}[
    language=Python,
    basicstyle=\small\ttfamily,
    keywordstyle=\color{blue!70!black}\bfseries,
    commentstyle=\color{green!60!black}\itshape,
    stringstyle=\color{purple!70!black},
    numberstyle=\tiny\color{gray!70!black},
    numbers=left,
    numbersep=8pt,
    breaklines=true,
    showstringspaces=false,
    frame=none,
    tabsize=4,
    captionpos=b,
    backgroundcolor=\color{gray!3}
]
import numpy as np
def priority(item: float, bins: np.ndarray) -> np.ndarray:
    """Returns priority with which we want to add item to each bin.
    Args:
        item: Size of item to be added to the bin.
        bins: Array of capacities for each bin.
    Return:
        Array of same size as bins with priority score of each bin.
    """
    return item - bins
\end{lstlisting}
\end{mdframed}
\end{figure}

\subsection{Traveling Salesman Problem (TSP)} 

\subsubsection{Problem Definition:} 
    Let $G = (V, E)$ be a complete graph where: $V = \{v_1, \dots, v_n\}$ represents $n$ cities with coordinates $\mathbf{x}_i \in [0,1]^2$ and $E$ contains edges with costs $c_{ij} = \|\mathbf{x}_i - \mathbf{x}_j\|_2$ (we consider Euclidean distance).

    The objective is to find a Hamiltonian cycle $\pi = (\pi_1, \dots, \pi_n, \pi_1)$ minimizing:
    $$
    \mathcal{L}_{\text{TSP}} = \sum_{k=1}^{n-1} c_{\pi_k \pi_{k+1}} + c_{\pi_n \pi_1}.
    $$
    
\subsubsection{Designed Heuristic:}

We target a construction heuristic that iteratively builds the tour by selecting the next city. The designed heuristic is the ID of the next city, given the current city, unvisited cities and the distance matrix. Performance is measured by the average relative gap to LKH-3~\cite{helsgaun2017extension}:

\subsubsection{Data Generation:}
    \begin{itemize}
        \item \textbf{Training Instances}: 128 instances where cities are arranged in clustered patterns. These distributions generate between $10$-$200$ cities across $32$ instances (each configuration 32 instances), with four different clustering configurations that vary in both the number of clusters ($3$ or $10$) and the standard deviation of points within each cluster ($0.03$ or $0.07$). The distribution is implemented by first generating random cluster centers within a bounded region ($0.2$--$0.8$), then sampling city locations from normal distributions centered at these points in $[0,1]^2$, with the specified standard deviation controlling how tightly or loosely cities are grouped around their cluster centers~\cite{bi2022learning}.

        \item \textbf{Testing Instances:} 80 instances with city counts of 50, 100, 200, 500, and 1000. 16 instances per size where city coordinates are uniformly distributed in $[0,1]^2$. 
 \end{itemize}

% TSP aims to find a route that minimizes the total traveling distance for a salesman required to visit each city exactly once before returning to the starting point. We investigate the constructive heuristic design for TSP. Specifically, we adopt an iteratively constructive framework to start from one node and iteratively select the next node until all nodes have been selected and back to the start node. The task is to design a heuristic for choosing the next node to minimize the route length.

% For training, we generated 128 instances using a clustered distribution approach inspired by the Solomon benchmark dataset and following~\cite{bi2022learning}. The instances were created using a Gaussian cluster distribution where city coordinates are confined to the unit square $[0,1]$. We systematically varied both the number of cities (ranging from 10 to 200) and cluster configurations, producing 32 instances for each of four distribution variants: 3 clusters with standard deviation 0.3, 10 clusters with standard deviation 0.3, 3 clusters with standard deviation 0.7, and 10 clusters with standard deviation 0.7. Cluster centers were randomly generated within the $[0.2,0.8]$ range, with cities distributed as evenly as possible among clusters. %This approach provides a diverse set of problem instances that better reflect real-world geographic clustering compared to uniform random distributions.

The task description and template used are as follows:

\begin{tcolorbox}[colback=gray!5, colframe=gray!50]
\textbf{TSP Task Description:} Given a set of nodes with their coordinates, you need to find the shortest route that visits each node once and returns to the starting node. The task can be solved step-by-step by starting from the current node and iteratively choosing the next node. Help me design a novel algorithm that is different from the algorithms in literature to select the next node in each step.
\end{tcolorbox}

\begin{figure}[ht]
\centering
\begin{mdframed}[
    linecolor=gray!50,
    linewidth=1pt,
    roundcorner=5pt,
    backgroundcolor=gray!5,
    innertopmargin=10pt,
    innerbottommargin=10pt,
    innerrightmargin=10pt,
    innerleftmargin=10pt,
    frametitlebackgroundcolor=gray!20,
    frametitleaboveskip=3pt,
    frametitlebelowskip=6pt,
    frametitlefont=\bfseries\sffamily
]
\begin{lstlisting}[
    language=Python,
    basicstyle=\small\ttfamily,
    keywordstyle=\color{blue!70!black}\bfseries,
    commentstyle=\color{green!60!black}\itshape,
    stringstyle=\color{purple!70!black},
    numberstyle=\tiny\color{gray!70!black},
    numbers=left,
    numbersep=8pt,
    breaklines=true,
    showstringspaces=false,
    frame=none,
    tabsize=4,
    captionpos=b,
    backgroundcolor=\color{gray!3}
]
import numpy as np
def select_next_node(current_node: int, destination_node: int, unvisited_nodes: np.ndarray, distance_matrix: np.ndarray) -> int: 
    """
    Design a novel algorithm to select the next node in each step.

    Args:
    current_node: ID of the current node.
    destination_node: ID of the destination node.
    unvisited_nodes: Array of IDs of unvisited nodes.
    distance_matrix: Distance matrix of nodes.

    Return:
    ID of the next node to visit.
    """
    next_node = unvisited_nodes[0]
    return next_node
\end{lstlisting}
\end{mdframed}
\end{figure}

\subsection{Capacitated Vehicle Routing Problem (CVRP)} 

\subsubsection{Problem Definition:}

CVRP aims to minimize the total traveling distances of a fleet of vehicles given a depot and a set of customers with coordinates and demands. Given: 1) Depot $v_0$ and customers $\{v_1,...,v_n\}$ with coordinates $\mathbf{x}_i \in [0,1]^2$, 2) Demands $d_i \in \mathbb{Z}^+$ ($d_0=0$), 3) Vehicle capacity $Q \in \mathbb{Z}^+$, 4) Distance metric $c_{ij} = \|\mathbf{x}_i - \mathbf{x}_j\|_2$, find routes $\mathcal{R} = \{r_1,...,r_m\}$. Each route $r_k$ starts/ends at $v_0$. Capacity constraints are satisfied $\sum_{v_i \in r_k} d_i \leq Q$ and all customers served exactly once. The objective is to minimize total distance.

\subsubsection{Designed Heuristic:}
Evaluation uses the same gap metric as TSP against 
    
\subsubsection{Data Generation:}
    \begin{itemize}
        \item \textbf{Training Instances:} 256 instances with the number of nodes $n \sim \mathcal{U}\{20,200\}$ nodes (including depot), capacity $Q \sim \mathcal{U}\{10,150\}$, and demands $d_i \sim \mathcal{U}\{1,10\}$. The coordinates are generated using uniform distribution in $[0,1]^2$

        \item \textbf{Testing Instances:} 128 test instances are generated with $n_c \in \{50, 100, 200, 500\}$ cities, where each instance contains uniformly distributed coordinates in $[0,1]^2$, random demands $d_i \sim \mathcal{U}\{1,10\}$, and vehicle capacities $Q \sim \mathcal{U}\{40,150\}$. For each city size, 32 instances are created.
    \end{itemize}

% CVRP aims to minimize the total traveling distances of a fleet of vehicles given a depot and a set of customers with coordinates and demands. The problem is constrained by: (1) The vehicles start from the depot and return to the depot; (2) Each customer should be visited only once; (3) All the demands should be satisfied while the capacity of the vehicle should not be exceeded. Similar to TSP, we adopt an iteratively constructive framework to start from one node and iteratively select the next node until all nodes have been selected and return to the depot. The task is to design a heuristic for selecting the next node to minimize the total route length with all constraints satisfied.

% For training, we generate 128 instances using a uniform distribution sampled in the unit square [0, 1]. The number of nodes and capacity are random sampled in $[20,200]$ and $[40,150]$, respectively. We do not use the Gaussion distribution as that in TSP instances, because the size and capacities have already brought diverse to the distribution.

The task description and template used are as follows:

\begin{tcolorbox}[colback=gray!5, colframe=gray!50]
\textbf{CVRP Task Description:} Given a set of customers and a fleet of vehicles with limited capacity, the task is to design a novel algorithm to select the next node in each step, with the objective of minimizing the total cost.
\end{tcolorbox}

\begin{figure}[htbp]
\centering
\begin{mdframed}[
    linecolor=gray!50,
    linewidth=1pt,
    roundcorner=5pt,
    backgroundcolor=gray!5,
    innertopmargin=10pt,
    innerbottommargin=10pt,
    innerrightmargin=10pt,
    innerleftmargin=10pt,
    frametitlebackgroundcolor=gray!20,
    frametitleaboveskip=3pt,
    frametitlebelowskip=6pt,
    frametitlefont=\bfseries\sffamily
]
\begin{lstlisting}[
    language=Python,
    basicstyle=\small\ttfamily,
    keywordstyle=\color{blue!70!black}\bfseries,
    commentstyle=\color{green!60!black}\itshape,
    stringstyle=\color{purple!70!black},
    numberstyle=\tiny\color{gray!70!black},
    numbers=left,
    numbersep=8pt,
    breaklines=true,
    showstringspaces=false,
    frame=none,
    tabsize=4,
    captionpos=b,
    backgroundcolor=\color{gray!3}
]
import numpy as np
def select_next_node(current_node: int, depot: int, unvisited_nodes: np.ndarray, rest_capacity: np.ndarray, demands: np.ndarray, distance_matrix: np.ndarray) -> int:
    """Design a novel algorithm to select the next node in each step.
    Args:
        current_node: ID of the current node.
        depot: ID of the depot.
        unvisited_nodes: Array of IDs of unvisited nodes.
        rest_capacity: rest capacity of vehicle
        demands: demands of nodes
        distance_matrix: Distance matrix of nodes.
    Return:
        ID of the next node to visit.
    """
    best_score = -1
    next_node = -1
    for node in unvisited_nodes:
        demand = demands[node]
        distance = distance_matrix[current_node][node]
        if demand <= rest_capacity:
 score = demand / distance 
 if score > best_score:
     best_score = score
     next_node = node
    return next_node
\end{lstlisting}
\end{mdframed}
\end{figure}

\newpage

\section{Method Details}

We present the detailed prompts used in EoH-S. We have three different prompt engineering approaches for: 1) generating heuristics for the initial population, 2) diversity-aware search, and 3) local search.

I'll enhance the illustration by using different colours for each part of the prompt where:
\begin{itemize}
    \item \textcolor{purple}{\{Task description\}} and \textcolor{purple}{\{Task template\}} are the task description and template introduced in the last task section.
    \item \textcolor{teal}{\{Heuristics\}} are selected parent heuristic(s).

    \item \textbf{Black sentences} are instructions for generating thoughts and codes.

    \item \textcolor{brown}{Broan sentences} are strategy-specific prompts for complementary-aware search and local search.

    \item \textcolor{gray!70}{Gray words} are additional instructions for robust and efficient inference.

\end{itemize}

\subsection{Initialization}

The initialization phase generates the initial population of heuristic algorithms through the following prompt:

\begin{tcolorbox}[colback=blue!5, colframe=blue!50, title=\textbf{Initialization Prompt}, fonttitle=\bfseries]
\textcolor{purple}{\{task description\}}

First, describe your new algorithm and main steps in one sentence. The description must be inside within boxed \textcolor{orange}{\{\{\}\}}.

Next, implement the following Python function:\\
\textcolor{purple}{\{task template\}}

\textcolor{gray!70}{Do not give additional explanations.}
\end{tcolorbox}

This prompt instructs the language model to generate a concise description of a new algorithm along with its implementation according to the specified function template. The description is required to be enclosed within double curly braces for easy extraction.

\subsection{Diversity-aware Search}

The diversity-aware search phase aims to generate new algorithms that are distinct from existing ones in the population:

\begin{tcolorbox}[colback=green!5, colframe=green!50, title=\textbf{Diversity-aware Search Prompt}, fonttitle=\bfseries]
\textcolor{purple}{\{task description\}}

I have 2 existing algorithms with their codes as follows:\\
\textcolor{teal}{\{heuristics\}}\\

\textcolor{brown}{These algorithms are effective for solving different instance distributions. Please help me create a new algorithm that is different from the given ones.} \\

First, describe your new algorithm and main steps in one sentence. The description must be inside within boxed \textcolor{orange}{\{\{\}\}}.

Next, implement the following Python function:\\
\textcolor{purple}{\{task template\}}

\textcolor{gray!70}{Do not give additional explanations.}
\end{tcolorbox}

This prompt explicitly requests the language model to create an algorithm that differs from the existing ones in the population, promoting diversity in the search space. The existing algorithms are provided as context to guide the model toward unexplored algorithmic approaches.

\subsection{Local Search}

The local search phase focuses on refining and improving a specific algorithm:

\begin{tcolorbox}[colback=red!5, colframe=red!50, title=\textbf{Local Search Prompt}, fonttitle=\bfseries]
\textcolor{purple}{\{task description\}}

I have one algorithm with its code as follows:\\
\textcolor{teal}{\{heuristic\}}\\

\textcolor{brown}{Please assist me in creating an improved version of the algorithm provided.} \\

First, describe your new algorithm and main steps in one sentence. The description must be inside within boxed \textcolor{orange}{\{\{\}\}}. 

Next, implement the following Python function:\\
\textcolor{purple}{\{task template\}}

\textcolor{gray!70}{Do not give additional explanations.}
\end{tcolorbox}

\section{More Results}

\subsection{Convergence Curve and Complementary Performance}
Figure~\ref{convergence_obp} analyzes the convergence behaviour of the heuristic set designed using EoH-S on OBP. The upper section tracks performance progression as the number of samples increases, illustrating how the ensemble leverages complementary heuristics to improve solution quality. The lower section visualizes this complementary behavior through a radar plot, capturing the dominance of individual heuristics across three key convergence phases (initial, intermediate, and final). Each radar dimension corresponds to a problem instance (total: 256) and is colored by the top-performing heuristic (out of 10) for that instance. The legend quantifies heuristic dominance by reporting the number of instances where each heuristic ranked first, highlighting their specialized roles across the instance distribution. Figure~\ref{convergence_cvrp_eoh_reevo} shows the results of EoH and ReEvo on CVRP.

These figures demonstrate that EoH-S effectively evolves a set of complementary heuristics, with the following key observations:

\begin{itemize}

\item \textbf{EoH-S achieves stable convergence.}
The complementary population management in EoH-S adheres to Theorem 3 (see main paper), ensuring convergence within 2,000 samples. In contrast, as shown in Figure~\ref{convergence_cvrp_eoh_reevo}, both EoH and ReEvo exhibit fluctuations in complementary performance during optimization. While the single best heuristic (red curve) converges in these methods, their heuristic sets lack guaranteed complementary performance. Notably, ReEvo’s final top-10 heuristics underperform compared to those found mid-search, highlighting its instability in maintaining complementary behavior.

\item \textbf{EoH-S generates a complementary heuristic set.}  
On both OBP and CVRP, EoH-S progressively improves heuristic diversity and performance. Initially, one or two heuristics dominate, but EoH-S systematically explores new heuristics to enhance both individual and collective performance. In the final set, all 10 heuristics contribute meaningfully across the 256 instances. In contrast, EoH retains two fully dominated heuristics, and ReEvo’s top one heuristic dominates nearly half the instances (119/256). Although ReEvo’s mid-search heuristics show better complementarity, its focus on optimizing only the single best heuristic (red curve) fails to ensure sustained improvement for the entire set.  

\item \textbf{Complementary behaviour varies by problem domain.} 
EoH-S produces a more balanced heuristic set for CVRP than for OBP. On OBP, one heuristic dominates most instances initially, and the top two heuristics still cover over half the instances in the final set. This disparity arises because the 256 CVRP instances with varying capacities and sizes are inherently more diverse and challenging than OBP instances. These results underscore the importance of complementary heuristic design while also revealing a limitation: EoH-S may be less effective when instance diversity is low (e.g., if a single heuristic suffices for all instances).  
\end{itemize}

\begin{figure}[htbp]
    \centering
    \includegraphics[width=1\linewidth]{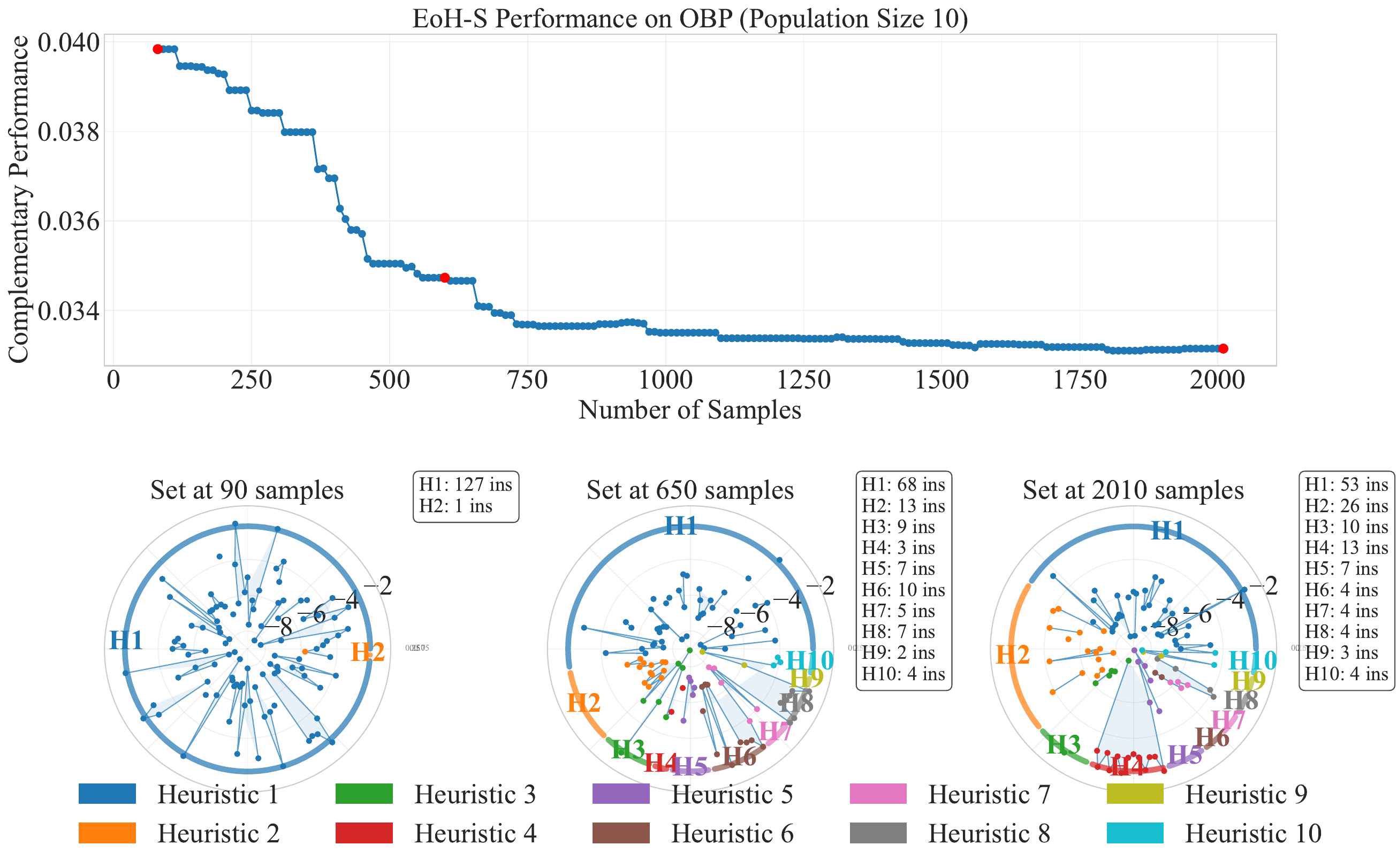}
    \caption{Convergence of EoH-S on OBP: The upper section depicts the performance w.r.t., the number of samples, while the lower section presents a radar plot characterizing the complementary behaviour at three distinct phases of the convergence process (initial, intermediate, and final). The radar plot dimensions represent the best-performing heuristic among the 10 heuristics in the heuristic set for each of the 128 problem instances, with each dimension colored according to the identity of the top-performing heuristic. The accompanying legend reports the number of instances for which each heuristic achieved the highest ranking, providing insight into their complementary behaviour on the 128 diverse instances.}
~\label{convergence_obp}
\end{figure}

\begin{figure}[htbp]
    \centering
    \includegraphics[width=0.95\linewidth]{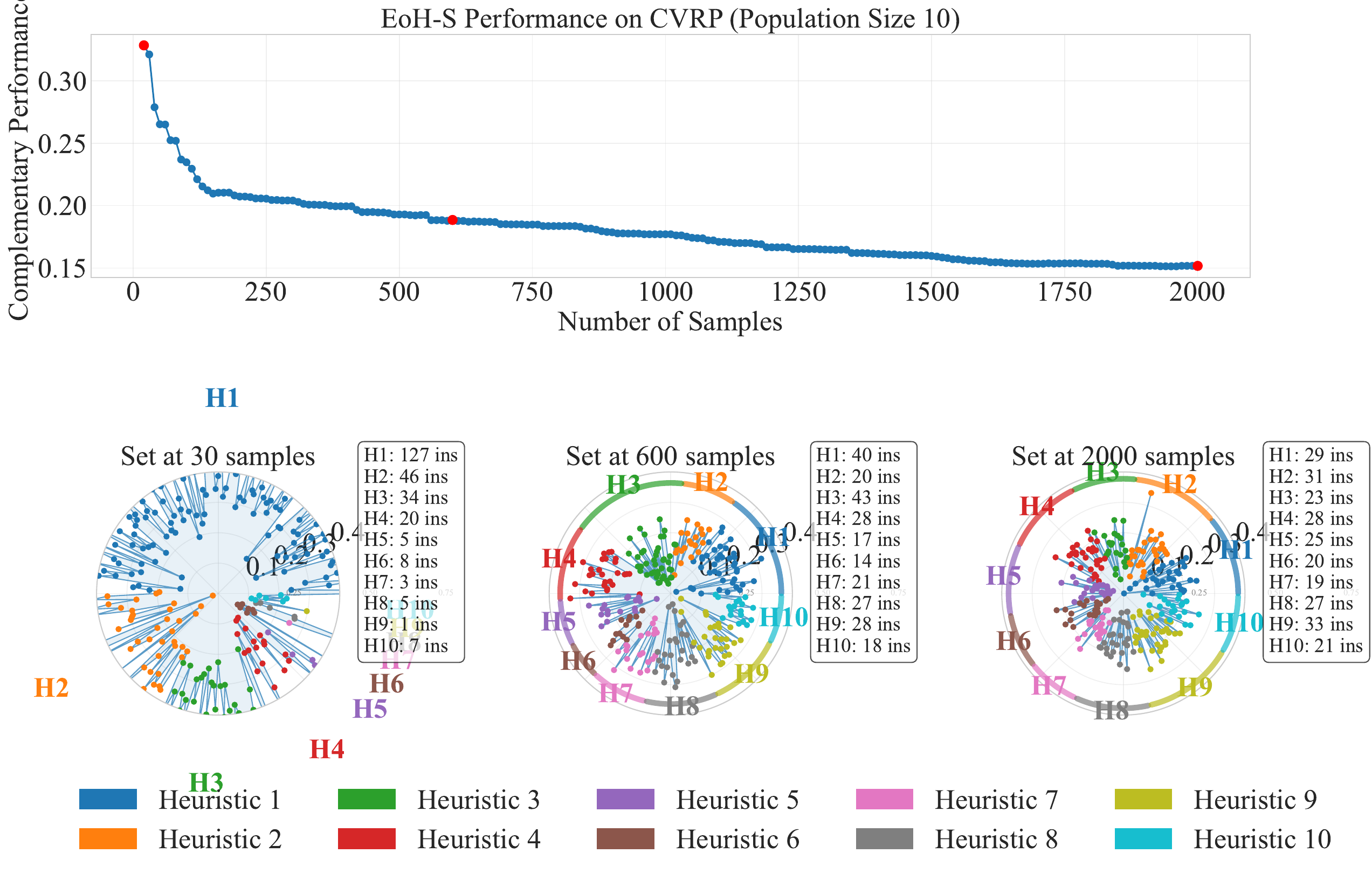}
    \caption{Convergence of EoH-S on CVRP: The upper section depicts the performance w.r.t., the number of samples, while the lower section presents a radar plot characterizing the complementary behaviour at three distinct phases of the convergence process (initial, intermediate, and final). The radar plot dimensions represent the best-performing heuristic among the 10 heuristics in the heuristic set for each of the 256 problem instances, with each dimension colored according to the identity of the top-performing heuristic. The accompanying legend reports the number of instances for which each heuristic achieved the highest ranking, providing insight into their complementary behaviour on the 256 diverse instances.}
~\label{convergence_cvrp_eohs}
\end{figure}

\begin{figure}
    \centering
    \begin{subfigure}[b]{0.85\textwidth}
        \includegraphics[width=\textwidth]{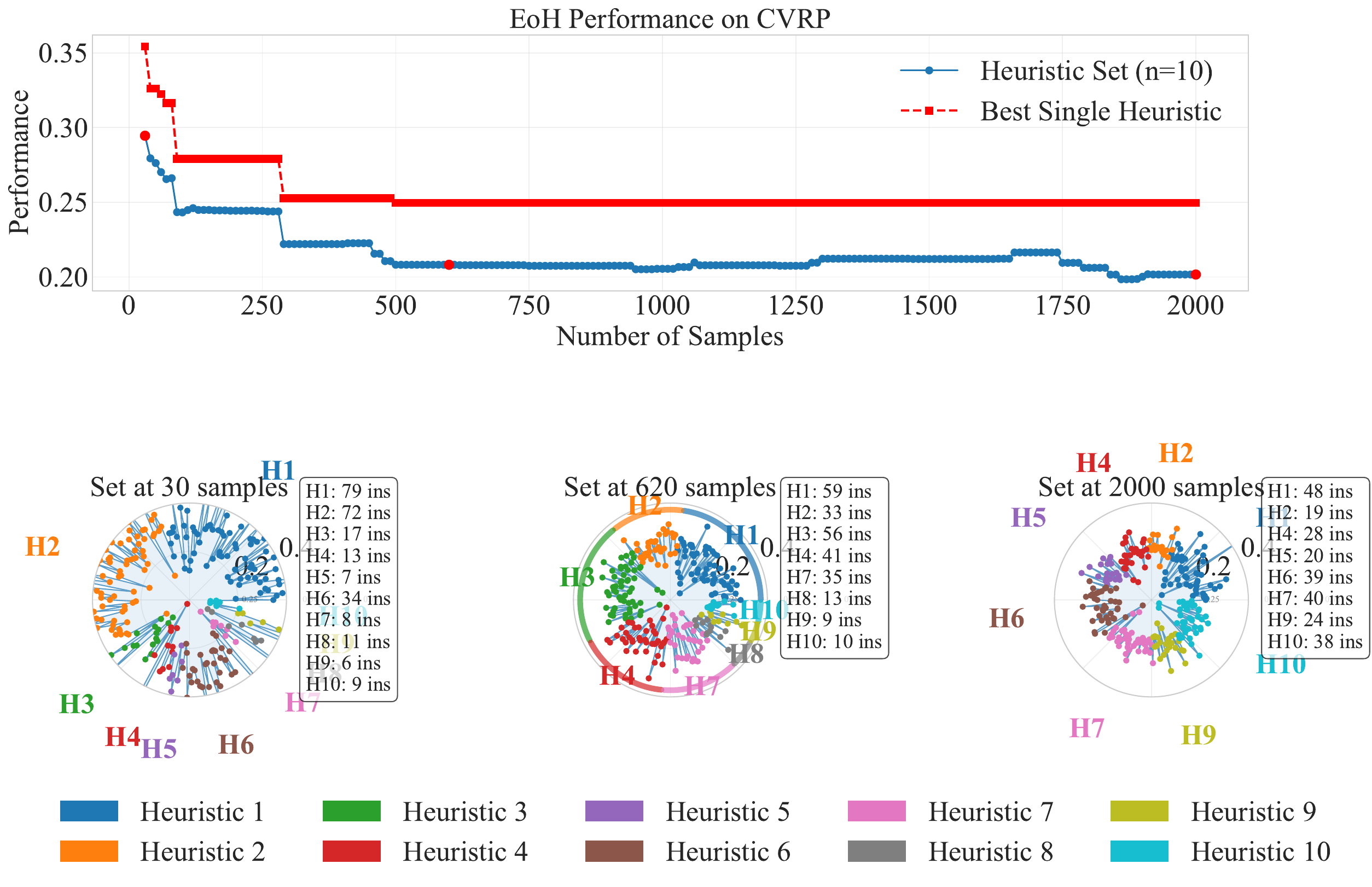}
        %\caption{Convergence of EoH on CVRP}
        \label{fig:eoh_cvrp}
    \end{subfigure}
    \hfill
    \begin{subfigure}[b]{0.85\textwidth}
        \includegraphics[width=\textwidth]{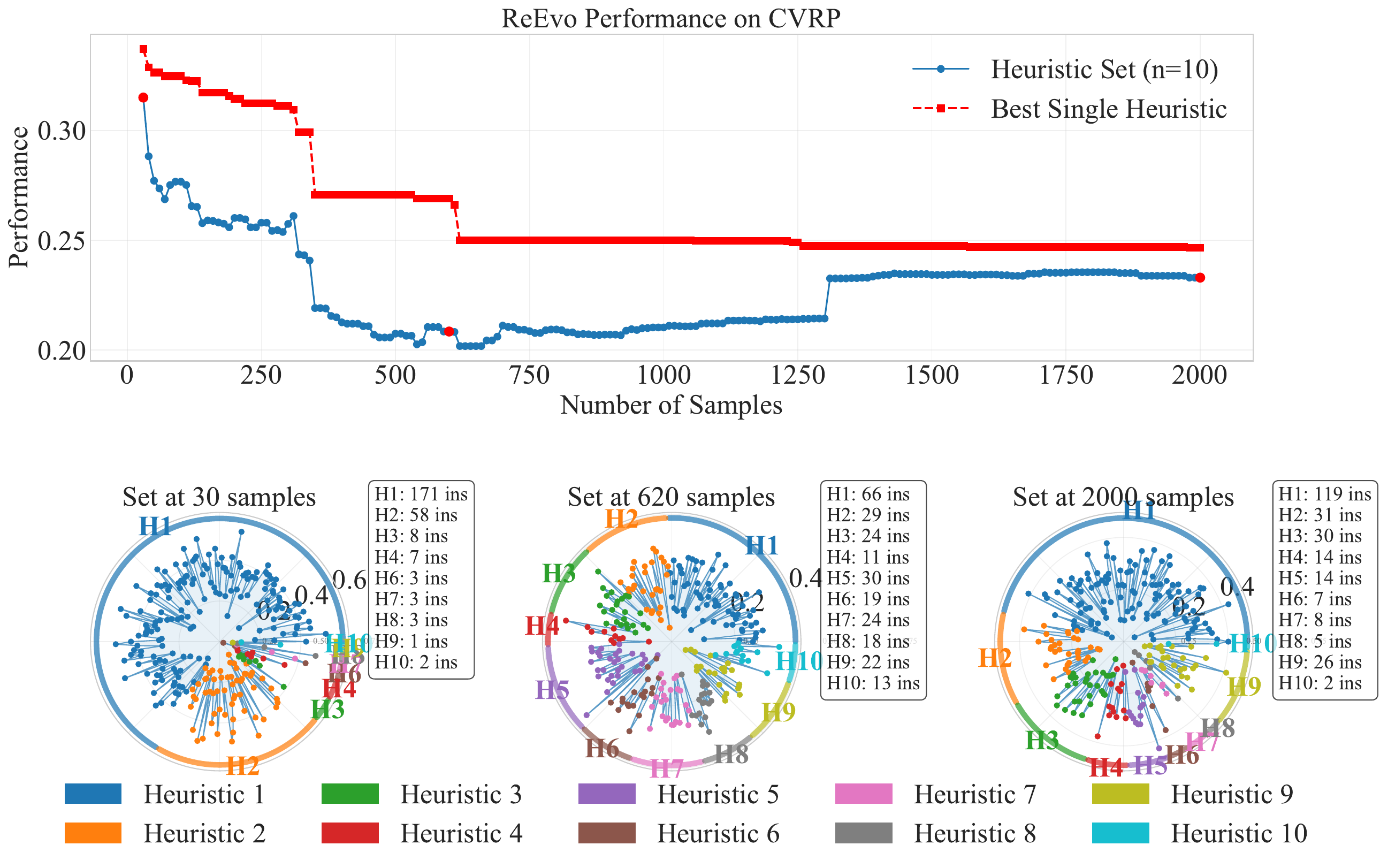}
        %\caption{Convergence of ReEvo on CVRP}
        \label{fig:reevo_cvrp}
    \end{subfigure}
    \caption{Convergence of EoH and ReEvo on CVRP: The upper section depicts the performance w.r.t., the number of samples, while the lower section presents a radar plot characterizing the complementary behaviour at three distinct phases of the convergence process (initial, intermediate, and final). The red curve indicates the convergence of the single current best heuristic in terms of average performance, while the blue curve is the complementary performance of the top 10 heuristics. The radar plot dimensions represent the best-performing heuristic among the 10 heuristics in the heuristic set for each of the 256 problem instances, with each dimension colored according to the identity of the top-performing heuristic.}
    \label{convergence_cvrp_eoh_reevo}
\end{figure}

\subsection{}

\begin{figure}[ht]
\centering
\begin{tcolorbox}[
    enhanced,
    colback=white,
    colframe=blue!70!black,
    arc=2mm,
    boxrule=0.8pt,
    title=EoH on OBP,
    fonttitle=\bfseries\color{white},
    coltitle=blue!70!black,
    interior style={top color=blue!5!white, bottom color=blue!10!white},
    drop shadow={opacity=0.3}
]
\begin{lstlisting}[
    language=Python,
    basicstyle=\small\ttfamily,
    keywordstyle=\color{blue},
    commentstyle=\color{green!50!black},
    stringstyle=\color{red!50!black},
    numberstyle=\tiny\color{gray},
    numbers=left,
    numbersep=5pt,
    breaklines=true,
    showstringspaces=false,
    frame=none,
    tabsize=4,
    captionpos=b
]
def priority(item: float, bins: np.ndarray) 
    remaining = bins - item
    valid = remaining >= 0
    fill_ratio = np.where(bins > 0, item / bins, 0)
    remaining_norm = np.where(bins > 0, remaining / np.max(bins), 0)
    weight = 1.0 - (bins / np.max(bins)) ** 2
    priority_scores = np.where(valid, fill_ratio * 0.3 + remaining_norm * 0.7 * weight, -np.inf)
    return priority_scores
\end{lstlisting}
\end{tcolorbox}

\begin{tcolorbox}[
    enhanced,
    colback=white,
    colframe=blue!70!black,
    arc=2mm,
    boxrule=0.8pt,
    title=ReEvo on OBP,
    fonttitle=\bfseries\color{white},
    coltitle=blue!70!black,
    interior style={top color=blue!5!white, bottom color=blue!10!white},
    drop shadow={opacity=0.3}
]
\begin{lstlisting}[
    language=Python,
    basicstyle=\small\ttfamily,
    keywordstyle=\color{blue},
    commentstyle=\color{green!50!black},
    stringstyle=\color{red!50!black},
    numberstyle=\tiny\color{gray},
    numbers=left,
    numbersep=5pt,
    breaklines=true,
    showstringspaces=false,
    frame=none,
    tabsize=4,
    captionpos=b
]
def priority(item: float, bins: np.ndarray) -> np.ndarray:
    remaining_space = bins - item
    valid_mask = remaining_space >= 0
    
    if not np.any(valid_mask):
        return np.zeros_like(bins)
    
    # Moderate exact fit boost
    exact_fit = (remaining_space == 0).astype(float) * 1e4
    
    # Adaptive exponent based on global fill ratio
    residuals = remaining_space[valid_mask] / bins[valid_mask]
    global_fill = 1 - np.mean(residuals)
    exponent = 3 + 8 * global_fill  # Smoothly varies from 3 to 11
    
    # Core penalty terms
    tight_fit = 1 / (1e-6 + residuals**exponent)
    slack_penalty = 1 / (1e-6 + (1 - residuals)**exponent)
    
    # Gradual blending with optimized transition
    transition_point = 0.6 + 0.2 * global_fill  # Adaptive transition
    blend_weight = 1 / (1 + np.exp(-15*(global_fill-transition_point)))
    
    # Simple harmonic mean
    harmonic = (2 * tight_fit * slack_penalty) / (1e-6 + tight_fit + slack_penalty)
    
    # Final priorities
    priorities = np.zeros_like(bins)
    priorities[valid_mask] = (
        exact_fit[valid_mask] +
        blend_weight * tight_fit +
        (1 - blend_weight) * harmonic
    )
    
    return priorities
\end{lstlisting}
\end{tcolorbox}
\caption{Heuristics designed by EoH and ReEvo for OBP.}
\end{figure}

\begin{figure}[ht]
\centering
\begin{tcolorbox}[
    enhanced,
    colback=white,
    colframe=blue!70!black,
    arc=2mm,
    boxrule=0.8pt,
    title=EoH-S on OBP (2 Complimentary Heuristics),
    fonttitle=\bfseries\color{white},
    coltitle=blue!70!black,
    interior style={top color=blue!5!white, bottom color=blue!10!white},
    drop shadow={opacity=0.3}
]
\begin{lstlisting}[
    language=Python,
    basicstyle=\small\ttfamily,
    keywordstyle=\color{blue},
    commentstyle=\color{green!50!black},
    stringstyle=\color{red!50!black},
    numberstyle=\tiny\color{gray},
    numbers=left,
    numbersep=5pt,
    breaklines=true,
    showstringspaces=false,
    frame=none,
    tabsize=4,
    captionpos=b
]
def priority(item: float, bins: np.ndarray) -> np.ndarray:
    capacity_ratio = np.percentile(bins, 80) / np.percentile(bins, 20)
    dynamic_target = 0.55 + 0.25 * (1 / (1 + np.exp(-3.0 * (capacity_ratio - 1.1))))
    remaining = bins - item
    valid = remaining >= 0
    filled = (bins - remaining) / bins
    # Enhanced multi-tier capacity scoring
    tier_weights = 0.85 + 0.5 * np.tanh(2.5 * (bins / np.percentile(bins, 65) - 1.2))
    optimal_zone = dynamic_target * (1 + 0.18 * np.sin(1.5 * np.pi * tier_weights))
    # Dual-phase attraction-repulsion
    attraction = np.exp(-5 * np.abs(filled - optimal_zone))
    repulsion = 1 - 0.35 * np.power(np.maximum(0, filled - (optimal_zone + 0.15)), 2.0)
    reward = 3.5 * attraction * repulsion * tier_weights
    # Density-adaptive dispersion
    density_gradient = np.abs(filled - np.median(filled[valid])) if np.any(valid) else 0
    dispersion = 1.4 - 0.4 * np.tanh(6 * density_gradient) * (1 - 0.25 * np.cos(3 * np.pi * filled))
    # Resonant harmonic modulation
    harmonic_factor = 1 + 0.2 * np.sin(3 * np.pi * (filled - dynamic_target + 0.05))
    priority_scores = np.where(valid,
                             -np.log1p(remaining) * reward * dispersion * harmonic_factor,
                             np.inf)
    return priority_scores

def priority(item: float, bins: np.ndarray) -> np.ndarray:
    remaining = bins - item
    valid = remaining >= 0
    normalized_remaining = np.where(bins > 0, remaining / bins, 0)
    current_fill = 1 - normalized_remainin
    # Enhanced dynamic target adaptation
    mean_bins = np.mean(bins)
    std_bins = np.std(bins) + 1e-9
    dynamic_target = 0.65 + 0.15 * (1 / (1 + np.exp(-mean_bins / (std_bins + 1e-9)))) + 0.02 * np.tanh(mean_bins)
    # Adaptive sigmoid scaling
    fill_deviation = current_fill - dynamic_target
    adaptive_scale = 7.0 + 5.0 * np.tanh(np.var(bins) / (mean_bins**2 + 1e-9))
    reward = 2.0 / (1 + np.exp(-fill_deviation * adaptive_scale)) + 1.0
    # Bin-size-aware logarithmic capacity utilization
    size_factor = np.log1p(bins / (mean_bins + 1e-9)) ** 0.35
    capacity_term = (1 - np.exp(-3.5 * normalized_remaining)) * (0.85 + 0.15 * size_factor)
    # Triple-phase stability mechanism
    global_fill = np.mean(current_fill[valid]) if np.any(valid) else dynamic_target
    historical_fill = np.mean(current_fill) if len(bins) > 1 else dynamic_target
    local_stability = 0.90 + 0.10 * np.exp(-15 * (current_fill - dynamic_target)**2)
    global_stability = 0.92 + 0.08 * np.exp(-20 * (current_fill - global_fill)**2)
    historical_stability = 0.94 + 0.06 * np.exp(-25 * (current_fill - historical_fill)**2)
    # Dynamic fragmentation penalty
    skewness = np.mean((bins - mean_bins)**3) / (std_bins**3 + 1e-9)
    kurtosis = np.mean((bins - mean_bins)**4) / (std_bins**4 + 1e-9) - 3
    fragmentation = (np.abs(normalized_remaining - 0.5) * 0.3 + 
                    np.abs(skewness) * 0.4 + 
                    np.abs(kurtosis) * 0.3) if np.any(valid) else 0
    
    priority_scores = np.where(valid, -capacity_term * reward * local_stability * global_stability * historical_stability / (1 + fragmentation), np.inf)
    return priority_scores
\end{lstlisting}
\end{tcolorbox}
\caption{Two heuristics selected from the heuristic set designed by EoH-S for OBP.}
\end{figure}

\subsection{Designed Heuristics}

We illustrate some example heuristics for EoH, ReEvo, and EoH-S. 

\begin{figure}[ht]
\centering
\begin{tcolorbox}[
    enhanced,
    colback=white,
    colframe=blue!70!black,
    arc=2mm,
    boxrule=0.8pt,
    title=EoH on TSP,
    fonttitle=\bfseries\color{white},
    coltitle=blue!70!black,
    interior style={top color=blue!5!white, bottom color=blue!10!white},
    drop shadow={opacity=0.3}
]
\begin{lstlisting}[
    language=Python,
    basicstyle=\small\ttfamily,
    keywordstyle=\color{blue},
    commentstyle=\color{green!50!black},
    stringstyle=\color{red!50!black},
    numberstyle=\tiny\color{gray},
    numbers=left,
    numbersep=5pt,
    breaklines=true,
    showstringspaces=false,
    frame=none,
    tabsize=4,
    captionpos=b
]
def select_next_node(current_node: int, destination_node: int, unvisited_nodes: np.ndarray, distance_matrix: np.ndarray) -> int:
    current_dist = distance_matrix[current_node, unvisited_nodes]
    dest_dist = distance_matrix[destination_node, unvisited_nodes]
    
    momentum = np.sum(distance_matrix[unvisited_nodes] - distance_matrix[current_node, unvisited_nodes].reshape(-1, 1), axis=1)
    degree_attraction = np.sum(distance_matrix[unvisited_nodes] > 0, axis=1)
    cluster_penalty = -np.mean(distance_matrix[:, unvisited_nodes], axis=0)
    exploitation = np.random.rand(len(unvisited_nodes))
    
    remaining_nodes = len(unvisited_nodes)
    total_nodes = len(distance_matrix)
    progress = remaining_nodes / total_nodes
    
    momentum_weight = np.tanh(progress)
    attraction_weight = 1 / (1 + np.exp(-5 * (1 - progress)))
    penalty_weight = np.tanh(1 - progress)
    exploitation_weight = 0.1 * progress
    
    combined_score = (momentum_weight * momentum) + (attraction_weight * degree_attraction) + (penalty_weight * cluster_penalty) + (exploitation_weight * exploitation)
    return unvisited_nodes[np.argmax(combined_score)]
\end{lstlisting}
\end{tcolorbox}
\begin{tcolorbox}[
    enhanced,
    colback=white,
    colframe=blue!70!black,
    arc=2mm,
    boxrule=0.8pt,
    title=ReEvo on TSP,
    fonttitle=\bfseries\color{white},
    coltitle=blue!70!black,
    interior style={top color=blue!5!white, bottom color=blue!10!white},
    drop shadow={opacity=0.3}
]
\begin{lstlisting}[
    language=Python,
    basicstyle=\small\ttfamily,
    keywordstyle=\color{blue},
    commentstyle=\color{green!50!black},
    stringstyle=\color{red!50!black},
    numberstyle=\tiny\color{gray},
    numbers=left,
    numbersep=5pt,
    breaklines=true,
    showstringspaces=false,
    frame=none,
    tabsize=4,
    captionpos=b
]
def select_next_node(current_node: int, destination_node: int, unvisited_nodes: np.ndarray, distance_matrix: np.ndarray) -> int:
    if len(unvisited_nodes) == 1:
        return unvisited_nodes[0]
    # Current proximity (harmonic to handle zero distances)
    current_dists = distance_matrix[current_node, unvisited_nodes]
    proximity = 1 / (current_dists + 1e-8)
    # Exact future potential via MST approximation
    future_potential = np.zeros(len(unvisited_nodes))
    for i, node in enumerate(unvisited_nodes):
        remaining_nodes = np.delete(unvisited_nodes, i)
        if not remaining_nodes.size:
            future_potential[i] = 0
            continue
        # Approximate remaining tour length using nearest neighbor distances
        remaining_dists = distance_matrix[node, remaining_nodes]
        future_potential[i] = np.mean(np.sort(remaining_dists)[:3])  # Top-3 nearest
    # Normalization with stability guarantees
    def safe_normalize(x):
        x = x - x.min()
        return x / (x.max() + 1e-8)
    p_norm = safe_normalize(proximity)
    fp_norm = safe_normalize(future_potential)
    # Adaptive weights using sigmoid transition
    progress = 1 - len(unvisited_nodes) / distance_matrix.shape[0]
    exploit_weight = 0.8 / (1 + np.exp(5*(progress - 0.6)))  # Sigmoid centered at 60%
    explore_weight = 0.2 * (1 - progress)**2  # Quadratic decay
    # Combined score with directional exploration
    base_score = 0.6*p_norm + 0.4*fp_norm
    noise = np.random.normal(0, 0.05, len(unvisited_nodes)) * explore_weight
    combined_score = exploit_weight * base_score + noise
    return unvisited_nodes[np.argmax(combined_score)]
\end{lstlisting}
\end{tcolorbox}
\caption{Heuristics designed by EoH and ReEvo on TSP.}
\label{fig:heuristics_tsp}
\end{figure}

\begin{figure}[ht]
\centering
\begin{tcolorbox}[
    enhanced,
    colback=white,
    colframe=blue!70!black,
    arc=2mm,
    boxrule=0.8pt,
    title=EoH-S on TSP,
    fonttitle=\bfseries\color{white},
    coltitle=blue!70!black,
    interior style={top color=blue!5!white, bottom color=blue!10!white},
    drop shadow={opacity=0.3}
]
\begin{lstlisting}[
    language=Python,
    basicstyle=\small\ttfamily,
    keywordstyle=\color{blue},
    commentstyle=\color{green!50!black},
    stringstyle=\color{red!50!black},
    numberstyle=\tiny\color{gray},
    numbers=left,
    numbersep=5pt,
    breaklines=true,
    showstringspaces=false,
    frame=none,
    tabsize=4,
    captionpos=b
]
def select_next_node(current_node: int, destination_node: int, unvisited_nodes: np.ndarray, distance_matrix: np.ndarray) -> int:
    current_dist = distance_matrix[current_node, unvisited_nodes]
    dest_dist = distance_matrix[destination_node, unvisited_nodes]
    progress = distance_matrix[current_node, destination_node] - dest_dist
    exploration_factor = np.log(len(unvisited_nodes) + 1) * (1 + np.random.rand() * 0.5)
    centrality = np.mean(distance_matrix[unvisited_nodes][:, unvisited_nodes], axis=1)
    penalty = np.maximum(0, dest_dist - np.percentile(dest_dist, 85))
    k = min(5, len(unvisited_nodes) - 1)
    if k > 0:
        sub_matrix = distance_matrix[np.ix_(unvisited_nodes, unvisited_nodes)]
        cluster_novelty = -np.mean(np.partition(sub_matrix, k, axis=1)[:, :k], axis=1)
    else:
        cluster_novelty = np.zeros(len(unvisited_nodes))
    if len(unvisited_nodes) < len(distance_matrix) - 1:
        last_move_dir = distance_matrix[unvisited_nodes, current_node] - distance_matrix[unvisited_nodes, destination_node]
        momentum = np.abs(last_move_dir - np.mean(last_move_dir)) * (1 + 0.1 * np.random.rand())
    else:
        momentum = np.zeros(len(unvisited_nodes))
    path_diversity = np.std(distance_matrix[unvisited_nodes], axis=1) * (1 + 0.1 * np.random.rand())
    remaining_path_heuristic = np.mean(distance_matrix[unvisited_nodes], axis=1) * (1 - 0.1 * np.random.rand())
    phase = len(unvisited_nodes) / len(distance_matrix)
    scale_factor = np.mean(distance_matrix) / np.max(distance_matrix)
    entropy = -np.sum(np.exp(-current_dist) * np.log(np.exp(-current_dist) + 1e-10))
    entropy_factor = 0.1 * entropy * (1 - phase)
    w_dist = (0.25 + (0.05 * phase)) * scale_factor
    w_progress = (0.15 - (0.05 * phase)) * scale_factor
    w_explore = (0.15 - (0.05 * phase)) * (1 - scale_factor) + entropy_factor
    w_centrality = (0.1 + (0.05 * phase)) * scale_factor
    w_penalty = (0.1 - (0.05 * phase)) * scale_factor
    w_novelty = (0.1 + (0.05 * phase)) * (1 - scale_factor)
    w_momentum = (0.05 * (1 - phase)) * scale_factor
    w_diversity = (0.05 * (1 - phase)) * (1 - scale_factor)
    w_heuristic = (0.05 * phase) * (1 - scale_factor)
    score = (w_dist * current_dist) + (w_progress * progress) + (w_explore * exploration_factor) - (w_centrality * centrality) + (w_penalty * penalty) + (w_novelty * cluster_novelty) + (w_momentum * momentum) + (w_diversity * path_diversity) + (w_heuristic * remaining_path_heuristic)
    return unvisited_nodes[np.argmin(score)]
    
\end{lstlisting}
\end{tcolorbox}
\caption{Heuristic designed by EoH-S for TSP.}
\end{figure}

\begin{figure}[ht]
\centering
\begin{tcolorbox}[
    enhanced,
    colback=white,
    colframe=blue!70!black,
    arc=2mm,
    boxrule=0.8pt,
    title=EoH on CVRP,
    fonttitle=\bfseries\color{white},
    coltitle=blue!70!black,
    interior style={top color=blue!5!white, bottom color=blue!10!white},
    drop shadow={opacity=0.3}
]
\begin{lstlisting}[
    language=Python,
    basicstyle=\small\ttfamily,
    keywordstyle=\color{blue},
    commentstyle=\color{green!50!black},
    stringstyle=\color{red!50!black},
    numberstyle=\tiny\color{gray},
    numbers=left,
    numbersep=5pt,
    breaklines=true,
    showstringspaces=false,
    frame=none,
    tabsize=4,
    captionpos=b
]
def select_next_node(current_node: int, depot: int, unvisited_nodes: np.ndarray, rest_capacity: np.ndarray, demands: np.ndarray, distance_matrix: np.ndarray) -> int:
    """Design a novel algorithm to select the next node in each step.
    Args:
        current_node: ID of the current node.
        depot: ID of the depot.
        unvisited_nodes: Array of IDs of unvisited nodes.
        rest_capacity: rest capacity of vehicle
        demands: demands of nodes
        distance_matrix: Distance matrix of nodes.
    Return:
        ID of the next node to visit.
    """
    feasible_nodes = unvisited_nodes[demands[unvisited_nodes] <= rest_capacity]
    if len(feasible_nodes) == 0:
        return depot
    
    current_distances = distance_matrix[current_node, feasible_nodes]
    depot_distances = distance_matrix[feasible_nodes, depot]
    distance_savings = (distance_matrix[current_node, depot] + distance_matrix[depot, feasible_nodes]) - current_distances
    
    demand_ratio = demands[feasible_nodes] / (rest_capacity + 1e-6)
    
    avg_distances = np.mean(distance_matrix[feasible_nodes][:, feasible_nodes], axis=1)
    density_ratio = current_distances / (avg_distances + 1e-6)
    
    capacity_factor = rest_capacity / (np.max(demands[feasible_nodes]) + 1e-6)
    proximity_factor = np.mean(current_distances) / (np.mean(depot_distances) + 1e-6)
    
    w1 = max(0.3, 0.7 - capacity_factor * 0.4)
    w2 = min(0.4, 0.2 + proximity_factor * 0.2)
    w3 = 1.0 - w1 - w2
    
    weights = w1 * distance_savings + w2 * demand_ratio + w3 * (1 / (density_ratio + 1e-6))
    next_node = feasible_nodes[np.argmax(weights)]
    return next_node
\end{lstlisting}
\end{tcolorbox}
\caption{Heuristic designed by EoH for CVRP.}
\end{figure}

\begin{figure}[ht]
\centering
\begin{tcolorbox}[
    enhanced,
    colback=white,
    colframe=blue!70!black,
    arc=2mm,
    boxrule=0.8pt,
    title=ReEvo on CVRP,
    fonttitle=\bfseries\color{white},
    coltitle=blue!70!black,
    interior style={top color=blue!5!white, bottom color=blue!10!white},
    drop shadow={opacity=0.3}
]
\begin{lstlisting}[
    language=Python,
    basicstyle=\small\ttfamily,
    keywordstyle=\color{blue},
    commentstyle=\color{green!50!black},
    stringstyle=\color{red!50!black},
    numberstyle=\tiny\color{gray},
    numbers=left,
    numbersep=5pt,
    breaklines=true,
    showstringspaces=false,
    frame=none,
    tabsize=4,
    captionpos=b
]
def select_next_node(current_node: int, depot: int, unvisited_nodes: np.ndarray, rest_capacity: np.ndarray, demands: np.ndarray, distance_matrix: np.ndarray) -> int:
    if not unvisited_nodes.size:
        return depot
    # Adaptive capacity buffer with demand characteristics and route progress
    demand_stats = demands[unvisited_nodes]
    demand_cv = np.std(demand_stats) / (np.mean(demand_stats) + 1e-10)
    route_progress = 1 - len(unvisited_nodes) / len(demands)
    # Three-component adaptive buffer
    buffer = (0.02 +  0.06 * (1 - np.exp(-3 * (demand_cv - 0.3))) + 0.03 * route_progress * (1 + 0.5 * demand_cv))
    feasible_mask = demands[unvisited_nodes] <= rest_capacity * (1 + buffer)
    feasible_nodes = unvisited_nodes[feasible_mask]
    if not feasible_nodes.size:
        return depot
    # Robust distance metrics using percentile normalization
    current_dists = distance_matrix[current_node, feasible_nodes]
    depot_dists = distance_matrix[feasible_nodes, depot]
    dist_p90 = np.percentile(distance_matrix, 90)
    # Normalized components with outlier protection
    norm_current = (current_dists - np.min(current_dists)) / (np.ptp(current_dists) + 1e-10)
    norm_depot = (depot_dists - np.min(depot_dists)) / (np.ptp(depot_dists) + 1e-10)
    # Route state analysis with multiple factors
    remaining_demand = np.sum(demands[unvisited_nodes])
    capacity_ratio = min(1.0, remaining_demand / (rest_capacity + 1e-10))
    urgency = (len(unvisited_nodes) / len(demands)) ** 0.7
    # Core scoring components with enhanced formulations
    proximity = (0.8 / (current_dists + 0.1*dist_p90) + 
                0.2 * (1 - norm_current) * (1 + 0.3 * urgency))
    utilization = np.power(demands[feasible_nodes], 0.8) / (rest_capacity + 1e-10)
    spatial_cohesion = np.exp(-3 * abs(norm_current - (1 - norm_depot)))
    # Dynamic weight adaptation using route state
    proximity_weight = 0.7 - 0.2 * (1 / (1 + np.exp(-12 * (capacity_ratio - 0.4))))
    utilization_weight = 0.6 / (1 + np.exp(-10 * (1.2 - capacity_ratio)))
    spatial_weight = 0.4 * (1 - np.exp(-3 * urgency))
    # Advanced spatial clustering analysis
    if len(feasible_nodes) > 1:
        centroid = np.mean(distance_matrix[feasible_nodes], axis=0)
        spatial_scores = np.linalg.norm(distance_matrix[feasible_nodes] - centroid, axis=1)
        spatial_scores = (spatial_scores - np.min(spatial_scores)) / (np.ptp(spatial_scores) + 1e-10)
    else:
        spatial_scores = np.zeros(len(feasible_nodes))
    # Adaptive critical demand detection
    demand_ratio = demands[feasible_nodes] / rest_capacity
    critical_threshold = (0.5 +  0.3 * np.tanh(5 * (demand_cv - 0.4)) +  0.15 * route_progress)
    critical_bonus = np.where(demand_ratio > critical_threshold, (demand_ratio - critical_threshold) ** 1.5, 0)
    # Balanced composite scoring
    scores = (proximity_weight * proximity + utilization_weight * utilization + spatial_weight * spatial_cohesion + 0.7 * spatial_scores + 1.5 * critical_bonus)
    # Sophisticated tie-breaking mechanism
    best_idx = np.argmax(scores)
    if np.sum(np.isclose(scores, scores[best_idx], rtol=1e-8, atol=1e-8)) > 1:
        tied_nodes = feasible_nodes[np.isclose(scores, scores[best_idx])]
        tie_breakers = np.column_stack([-critical_bonus[np.isclose(scores, scores[best_idx])], -spatial_cohesion[np.isclose(scores, scores[best_idx])], current_dists[np.isclose(scores, scores[best_idx])], -utilization[np.isclose(scores, scores[best_idx])]])
        return tied_nodes[np.lexsort(tie_breakers.T)[0]]
    return feasible_nodes[best_idx]
\end{lstlisting}
\end{tcolorbox}
\caption{Heuristic designed by ReEvo for CVRP.}
\end{figure}

\begin{figure}[ht]
\centering
\begin{tcolorbox}[
    enhanced,
    colback=white,
    colframe=blue!70!black,
    arc=2mm,
    boxrule=0.8pt,
    title=EoH-S on CVRP,
    fonttitle=\bfseries\color{white},
    coltitle=blue!70!black,
    interior style={top color=blue!5!white, bottom color=blue!10!white},
    drop shadow={opacity=0.3}
]
\begin{lstlisting}[
    language=Python,
    basicstyle=\small\ttfamily,
    keywordstyle=\color{blue},
    commentstyle=\color{green!50!black},
    stringstyle=\color{red!50!black},
    numberstyle=\tiny\color{gray},
    numbers=left,
    numbersep=5pt,
    breaklines=true,
    showstringspaces=false,
    frame=none,
    tabsize=4,
    captionpos=b
]
def select_next_node(current_node: int, depot: int, unvisited_nodes: np.ndarray, rest_capacity: np.ndarray, demands: np.ndarray, distance_matrix: np.ndarray) -> int:
    feasible_nodes = unvisited_nodes[demands[unvisited_nodes] <= rest_capacity]
    if len(feasible_nodes) == 0:
        return depot
    
    distances = distance_matrix[current_node, feasible_nodes]
    depot_distances = distance_matrix[feasible_nodes, depot]
    normalized_demands = demands[feasible_nodes] / np.max(demands[feasible_nodes])
    capacity_ratio = rest_capacity / np.max(demands[feasible_nodes])
    urgency = np.sum(demands[unvisited_nodes]) / (rest_capacity + 1e-6)
    entropy = np.std(distance_matrix[feasible_nodes][:, feasible_nodes]) / (np.mean(distance_matrix) + 1e-6)
    
    pheromone = np.exp(-(distances**0.75 + 1.2*depot_distances**0.65) / (1.3 * np.mean(distance_matrix)))
    swarm_intensity = 0.7 * (1 + np.tanh(2.1 - capacity_ratio**0.85)) * pheromone
    fuzzy_factor = 0.3 * (1 - np.exp(-entropy/(np.mean(distances) + 1e-6))) * (1 - 0.25*swarm_intensity)
    
    proximity_weight = 0.42 * (1 - 0.22 * np.exp(-capacity_ratio**0.9)) * swarm_intensity
    demand_weight = 0.36 * (1 + 0.55 * np.tanh(urgency**0.7)) * swarm_intensity
    neighborhood_weight = 0.15 * (1 - entropy**0.6) * (1 - 0.2*swarm_intensity)
    entropy_weight = 0.05 * np.exp(-np.std(distances)/(np.mean(distances) + 1e-6)) * fuzzy_factor
    adaptive_weight = 0.02 * (1 - np.exp(-np.std(normalized_demands)/(np.mean(normalized_demands) + 1e-6))) * fuzzy_factor
    
    proximity_scores = 1.25/(distances + 1e-6)**0.6 + 1.0/(depot_distances + 1e-6)**0.5
    demand_scores = normalized_demands**1.6 * proximity_scores
    neighborhood_scores = np.array([np.sum(distance_matrix[n][feasible_nodes]) for n in feasible_nodes]) / (distances + 1e-6)**0.3
    entropy_scores = (rest_capacity - demands[feasible_nodes])**0.85 * depot_distances / (distances + 1e-6)
    adaptive_scores = (0.45 + 0.55*np.random.rand(len(feasible_nodes)))**1.9 * (demands[feasible_nodes] / (distances + 1e-6)**0.4)
    
    combined_scores = (
        proximity_weight * proximity_scores +
        demand_weight * demand_scores +
        neighborhood_weight * neighborhood_scores +
        entropy_weight * entropy_scores +
        adaptive_weight * adaptive_scores
    )
    
    return feasible_nodes[np.argmax(combined_scores)]

    
\end{lstlisting}
\end{tcolorbox}
\caption{Heuristic designed by EoH-S for CVRP.}
\end{figure}

\clearpage
\newpage
\subsection{Detailed Results on Benchmark Sets}

\subsubsection{BPPLib}

We select several representative benchmark sets from BPPLib~\cite{delorme2018bpplib}. The chosen sets and their characteristics are summarized in Table~\ref{bpplib}, comprising over 700 instances with item counts ranging from 100 to 1,002. For consistency in evaluation, we normalize the bin capacity to 100. Due to space constraints, we omit detailed per-instance results, but the complete data will be available as supplementary materials.

\begin{table}[htbp]
\centering
\caption{BPPLib benchmark sets.}~\label{bpplib}
\begin{tabular}{cccc}
\toprule
Benchmark Set                 & Number of Instances & Capacity     & Number of Items \\
                 \midrule
Schwerin\_1~\cite{scholl1997bison}      & 100                 & 1k           & 100             \\
Schwerin\_2~\cite{scholl1997bison}      & 100                 & 1k           & 120             \\
AugmentedIRUP~\cite{delorme2016bin}    & 250                 & \{2.5K–80K\} & \{201-1002\}    \\
AugmentedNonIRUP~\cite{delorme2016bin} & 250                 & \{2.5K–80K\} & \{201-1002\}    \\
Scholl\_Hard~\cite{delorme2018bpplib}     & 28                  & 10000        & \{160-200\}   
\\
\bottomrule
\end{tabular}

\end{table}

\subsubsection{TSPLib}

We select commonly used 49 symmetric Euclidean TSPLib instances~\cite{reinelt1991tsplib}, with problem sizes ranging from 52 to 1,000 nodes. For consistent evaluation, we normalize all coordinates to the range $[0,1]^2$ by scaling both dimensions uniformly using the maximum spatial extent:

\begin{equation}
\text{Scaling factor} = \max(x_{\max}-x_{\min},\ y_{\max}-y_{\min})
\end{equation}

This approach preserves the original aspect ratio while ensuring all instances are comparably scaled. 

The detailed results on all instances are listed in Table~\ref{tsplib_results}. We report the relative gap to LKH~\cite{helsgaun2017extension}. Results show that EoH-S ranks the first on almost all the instances (46 out of 49).

% \usepackage{multirow}

\subsubsection{CVRPLib}

We select 7 commonly used benchmark sets, including A, B, E, F, M, P, and X, from CVRPLib~\cite{uchoa2017new}. The chosen sets and their
characteristics are summarized in Table~\ref{cvrplib}. Due to the time limit, we do not test on all instances from the X set. For consistent evaluation, we normalize all coordinates to the range $[0,1]^2$ by scaling both dimensions uniformly using the maximum spatial extent, similar to TSPLib. 

\begin{table}[htbp]
\centering
\caption{CVRPLib benchmark sets}~\label{cvrplib}
\begin{tabular}{ccc}
\toprule
Benchmark   Set & Number of Instances & Instance Size \\
\midrule
Set A           & 27                  & 31-79         \\
Set B           & 23                  & 30-77         \\
Set E           & 11                  & 22-101        \\
Set F           & 3                   & 44-134        \\
Set M           & 5                   & 100-199       \\
Set P           & 23                  & 15-100        \\
Set X           & 43                  & 100-300      \\
\bottomrule
\end{tabular}
\end{table}

The detailed results are presented in Tables~\ref{cvrplib_results_a_b}, \ref{cvrplib_results_efmp}, and \ref{cvrplib_results_x}. EoH-S consistently outperforms existing methods across benchmark instances of varying sizes and distributions, ranking first in over half of the instances and achieving the best average performance with significant improvements.

\begin{table*}[]
\centering
\tiny
\caption{Results on TSPLib instances}~\label{tsplib_results}
\resizebox{0.74\textwidth}{!}{
\begin{tabular}{llllllll}
\toprule
\multicolumn{1}{c}{\multirow{2}{*}{Instance}} & \multicolumn{2}{c}{Funsearch} & \multicolumn{2}{c}{EoH} & \multicolumn{2}{c}{Reevo} & \multicolumn{1}{c}{\multirow{2}{*}{EoH-S}} \\
\multicolumn{1}{c}{}     & top 1     & top 10 & top 1      & top 10     & top 1   & top 10& \multicolumn{1}{c}{}  \\
\midrule
a280& 0.197     & 0.186& 0.243      & 0.239      & 0.311   & 0.141& \textbf{0.137}        \\
berlin52    & 0.177     & 0.158& 0.163      & 0.163      & 0.206   & 0.112& \textbf{0.104}        \\
bier127     & 0.132     & 0.132& 0.174      & 0.161      & 0.155   & 0.155& \textbf{0.109}        \\
ch130       & 0.163     & 0.141& 0.158      & 0.132      & 0.281   & 0.160& \textbf{0.047}        \\
ch150       & 0.161     & 0.130& 0.213      & 0.149      & 0.242   & 0.203& \textbf{0.077}        \\
d198& 0.257     & \textbf{0.138}    & 0.150      & 0.150      & 0.188   & 0.157& 0.167    \\
d493& 0.148     & 0.148& 0.192      & 0.168      & 0.186   & 0.154& \textbf{0.123}        \\
d657& 0.174     & 0.173& 0.207      & 0.189      & 0.172   & 0.169& \textbf{0.153}        \\
eil51       & 0.098     & 0.077& 0.156      & 0.156      & 0.122   & 0.075& \textbf{0.044}        \\
eil76       & 0.164     & 0.125& 0.131      & 0.125      & 0.150   & 0.136& \textbf{0.063}        \\
eil101      & 0.163     & 0.134& 0.145      & 0.126      & 0.248   & 0.134& \textbf{0.098}        \\
fl417       & 0.257     & 0.161& 0.292      & 0.227      & 0.307   & 0.296& \textbf{0.147}        \\
gil262      & 0.200     & 0.200& 0.208      & 0.206      & 0.214   & 0.183& \textbf{0.115}        \\
kroA100     & 0.202     & 0.108& 0.130      & 0.123      & 0.342   & 0.193& \textbf{0.061}        \\
kroB100     & 0.176     & 0.166& 0.218      & 0.211      & 0.271   & 0.259& \textbf{0.115}        \\
kroC100     & 0.191     & 0.161& 0.177      & 0.154      & 0.180   & 0.180& \textbf{0.045}        \\
kroD100     & 0.179     & 0.173& 0.183      & 0.172      & 0.149   & 0.108& \textbf{0.092}        \\
kroE100     & 0.133     & 0.114& 0.189      & 0.173      & 0.232   & 0.199& \textbf{0.064}        \\
kroA150     & 0.211     & 0.168& 0.145      & 0.137      & 0.243   & 0.231& \textbf{0.101}        \\
kroB150     & 0.146     & 0.117& 0.116      & 0.116      & 0.283   & 0.265& \textbf{0.093}        \\
kroA200     & 0.142     & 0.122& 0.209      & 0.177      & 0.164   & 0.133& \textbf{0.089}        \\
kroB200     & 0.193     & 0.160& 0.203      & 0.203      & 0.153   & 0.099& \textbf{0.066}        \\
lin105      & 0.247     & 0.161& 0.149      & 0.146      & 0.447   & 0.179& \textbf{0.040}        \\
lin318      & 0.149     & 0.149& 0.121      & 0.121      & 0.323   & 0.200& \textbf{0.117}        \\
p654& 0.200     & 0.197& 0.291      & 0.265      & 0.210   & 0.210& \textbf{0.105}        \\
pcb442      & 0.145     & 0.136& 0.184      & 0.182      & 0.190   & 0.158& \textbf{0.102}        \\
pr76& 0.154     & 0.135& 0.156      & 0.153      & 0.330   & 0.070& \textbf{0.044}        \\
pr107       & 0.274     & 0.259& 0.142      & 0.132      & 0.028   & \textbf{0.028}  & 0.045    \\
pr124       & 0.237     & 0.116& 0.251      & 0.209      & 0.228   & 0.205& \textbf{0.058}        \\
pr136       & 0.136     & 0.134& 0.174      & 0.151      & 0.101   & 0.101& \textbf{0.059}        \\
pr144       & 0.086     & 0.081& 0.179      & 0.179      & 0.136   & 0.136& \textbf{0.051}        \\
pr152       & 0.277     & 0.222& 0.211      & 0.196      & 0.320   & 0.164& \textbf{0.123}        \\
pr226       & 0.183     & 0.139& 0.199      & 0.197      & 0.209   & 0.209& \textbf{0.107}        \\
pr264       & 0.217     & 0.189& 0.244      & 0.214      & 0.167   & 0.148& \textbf{0.084}        \\
pr299       & 0.219     & 0.184& 0.288      & 0.288      & 0.194   & 0.194& \textbf{0.116}        \\
pr439       & 0.261     & 0.203& 0.231      & 0.231      & 0.208   & 0.164& \textbf{0.120}        \\
rat99       & 0.144     & \textbf{0.127}    & 0.149      & 0.148      & 0.246   & 0.163& 0.128    \\
rat195      & 0.112     & 0.092& 0.071      & 0.071      & 0.099   & 0.096& \textbf{0.066}        \\
rat575      & 0.171     & 0.125& 0.152      & 0.152      & 0.219   & 0.182& \textbf{0.105}        \\
rat783      & 0.144     & 0.139& 0.204      & 0.189      & 0.261   & 0.204& \textbf{0.109}        \\
rd100       & 0.174     & 0.170& 0.158      & 0.158      & 0.292   & 0.203& \textbf{0.084}        \\
rd400       & 0.167     & 0.154& 0.156      & 0.148      & 0.219   & 0.195& \textbf{0.122}        \\
st70& 0.182     & 0.181& 0.100      & 0.095      & 0.184   & 0.184& \textbf{0.026}        \\
ts225       & 0.073     & \textbf{0.033}    & 0.101      & 0.101      & 0.192   & 0.116& \textbf{0.033}        \\
tsp225      & 0.149     & 0.103& 0.235      & 0.215      & 0.180   & 0.121& \textbf{0.097}        \\
u159& 0.222     & 0.215& 0.276      & 0.273      & 0.303   & 0.112& \textbf{0.083}        \\
u574& 0.243     & 0.213& 0.207      & 0.197      & 0.225   & 0.194& \textbf{0.135}        \\
u724& 0.241     & 0.199& 0.181      & 0.175      & 0.254   & 0.174& \textbf{0.123}        \\
pr1002      & 0.206     & 0.206& 0.226      & 0.220      & 0.222   & 0.209& \textbf{0.145}        \\
\midrule
Average     & 0.181     & 0.152& 0.184      & 0.173      & 0.220   & 0.165& \textbf{0.093}\\
\bottomrule
\end{tabular}
}
\end{table*}
% \usepackage{multirow}
\begin{table*}[]
\tiny
\centering
\caption{Results on CVRPLib instances (Set A and B)}~\label{cvrplib_results_a_b}
\renewcommand{\arraystretch}{0.9}
\resizebox{0.75\textwidth}{!}{
\begin{tabular}{llllllll}
\toprule
\multicolumn{1}{c}{\multirow{2}{*}{Instance}} & \multicolumn{2}{c}{FunSearch}     & \multicolumn{2}{c}{EoH}& \multicolumn{2}{c}{ReEvo}         & \multicolumn{1}{c}{\multirow{2}{*}{EoH-S}} \\
\multicolumn{1}{c}{}     & \multicolumn{1}{c}{Top 1} & \multicolumn{1}{c}{Top 10} & \multicolumn{1}{c}{Top 1} & \multicolumn{1}{c}{Top 10} & \multicolumn{1}{c}{Top 1} & \multicolumn{1}{c}{Top 10} & \multicolumn{1}{c}{}  \\
\midrule
A-n32-k5     & 0.472& 0.329& 0.329& 0.329& 0.357& 0.357& \textbf{0.309}       \\
A-n33-k5     & 0.229& 0.229& 0.294& 0.274& 0.224& 0.224& \textbf{0.198}       \\
A-n33-k6     & 0.203& \textbf{0.203}  & 0.254& \textbf{0.203}  & 0.331& 0.224& 0.224     \\
A-n34-k5     & 0.245& 0.184& 0.222& 0.222& 0.224& \textbf{0.161}  & 0.167     \\
A-n36-k5     & 0.381& 0.352& 0.386& 0.351& 0.382& \textbf{0.278}  & 0.320     \\
A-n37-k5     & 0.399& 0.364& 0.449& 0.377& 0.392& 0.352& \textbf{0.308}       \\
A-n37-k6     & 0.284& 0.281& 0.354& 0.279& 0.366& \textbf{0.147}  & 0.148     \\
A-n38-k5     & 0.389& 0.389& 0.356& \textbf{0.252}  & 0.339& 0.312& 0.256     \\
A-n39-k5     & 0.387& 0.336& 0.268& 0.231& 0.200& \textbf{0.194}  & 0.230     \\
A-n39-k6     & 0.332& 0.284& \textbf{0.249} & \textbf{0.249}  & 0.396& 0.285& 0.295     \\
A-n44-k6     & 0.347& 0.227& 0.199& 0.199& 0.219& 0.219& \textbf{0.184}       \\
A-n45-k6     & 0.606& 0.360& 0.315& 0.308& 0.463& 0.398& \textbf{0.191}       \\
A-n45-k7     & 0.220& 0.213& 0.212& 0.204& 0.195& 0.159& \textbf{0.120}       \\
A-n46-k7     & 0.367& 0.332& 0.505& 0.385& 0.388& \textbf{0.251}  & 0.285     \\
A-n48-k7     & 0.331& \textbf{0.267}  & 0.327& 0.307& 0.338& 0.300& 0.271     \\
A-n53-k7     & 0.435& 0.317& 0.244& 0.244& 0.380& 0.322& \textbf{0.201}       \\
A-n54-k7     & 0.234& 0.227& 0.215& 0.215& 0.329& 0.187& \textbf{0.136}       \\
A-n55-k9     & 0.317& 0.287& 0.369& 0.333& 0.248& \textbf{0.240}  & 0.278     \\
A-n60-k9     & 0.317& 0.317& 0.425& 0.310& 0.323& \textbf{0.305}  & 0.332     \\
A-n61-k9     & 0.506& 0.327& 0.466& 0.244& 0.447& 0.244& \textbf{0.230}       \\
A-n62-k8     & 0.368& 0.280& 0.329& 0.276& 0.334& 0.250& \textbf{0.238}       \\
A-n63-k9     & 0.312& 0.309& 0.286& 0.240& 0.182& 0.169& \textbf{0.145}       \\
A-n63-k10    & 0.486& 0.323& 0.334& 0.257& 0.420& 0.377& \textbf{0.222}       \\
A-n64-k9     & 0.377& 0.285& 0.333& 0.327& 0.322& 0.245& \textbf{0.186}       \\
A-n65-k9     & 0.441& 0.359& 0.388& \textbf{0.324}  & 0.468& 0.345& 0.342     \\
A-n69-k9     & 0.409& 0.298& 0.370& 0.297& 0.293& 0.253& \textbf{0.218}       \\
A-n80-k10    & 0.382& 0.258& 0.311& 0.252& 0.332& 0.245& \textbf{0.216}       \\
\midrule
B-n31-k5     & 0.239& 0.238& 0.150& 0.150& 0.086& 0.081& \textbf{0.073}       \\
B-n34-k5     & 0.095& 0.091& 0.157& 0.091& 0.269& 0.089& \textbf{0.088}       \\
B-n35-k5     & 0.316& 0.203& 0.203& 0.199& 0.305& \textbf{0.144}  & 0.200     \\
B-n38-k6     & 0.434& 0.431& 0.441& 0.285& 0.332& 0.332& \textbf{0.235}       \\
B-n39-k5     & 0.919& 0.915& 0.985& 0.778& 0.946& 0.517& \textbf{0.242}       \\
B-n41-k6     & 0.159& 0.159& 0.204& \textbf{0.121}  & 0.225& 0.195& 0.131     \\
B-n43-k6     & 0.291& 0.191& 0.210& 0.194& 0.204& \textbf{0.202}  & 0.208     \\
B-n44-k7     & 0.308& 0.298& 0.220& 0.220& 0.181& \textbf{0.175}  & 0.209     \\
B-n45-k5     & 0.363& 0.339& 0.283& 0.283& 0.315& 0.315& \textbf{0.196}       \\
B-n45-k6     & 0.510& 0.490& 0.433& 0.250& 0.593& 0.273& \textbf{0.219}       \\
B-n50-k7     & 0.327& 0.327& 0.476& 0.350& 0.404& 0.355& \textbf{0.237}       \\
B-n50-k8     & 0.142& 0.132& 0.192& 0.144& 0.229& 0.136& \textbf{0.110}       \\
B-n51-k7     & 0.304& 0.205& 0.310& 0.310& 0.333& 0.118& \textbf{0.114}       \\
B-n52-k7     & 0.661& 0.490& 0.551& 0.539& 0.737& 0.488& \textbf{0.184}       \\
B-n56-k7     & 0.768& 0.603& 0.734& 0.663& 0.508& 0.429& \textbf{0.229}       \\
B-n57-k7     & 0.368& 0.368& 0.450& 0.450& 0.406& 0.256& \textbf{0.187}       \\
B-n57-k9     & 0.224& 0.187& 0.176& 0.176& 0.164& 0.164& \textbf{0.162}       \\
B-n63-k10    & 0.427& 0.356& 0.307& 0.307& 0.388& 0.269& \textbf{0.200}       \\
B-n64-k9     & 0.491& 0.380& 0.626& 0.451& 0.431& 0.351& \textbf{0.192}       \\
B-n66-k9     & 0.283& 0.202& 0.251& 0.251& 0.192& 0.183& \textbf{0.114}       \\
B-n67-k10    & 0.536& 0.466& 0.394& 0.367& 0.413& 0.246& \textbf{0.239}       \\
B-n68-k9     & 0.250& 0.219& 0.282& 0.225& 0.195& 0.195& \textbf{0.164}       \\
B-n78-k10    & 0.362& 0.286& 0.575& 0.333& 0.244& \textbf{0.232}  & 0.282     \\
\midrule
Average      & 0.371& 0.310& 0.348& 0.293& 0.340& 0.256& \textbf{0.209}      \\
\bottomrule
\end{tabular}
}
\end{table*}

\begin{table*}[]
\tiny
\centering
\caption{Results on CVRPLib instances (Set E, F, M and P)}~\label{cvrplib_results_efmp}
\renewcommand{\arraystretch}{0.9}
\resizebox{0.75\textwidth}{!}{
\begin{tabular}{llllllll}
\toprule
\multicolumn{1}{c}{\multirow{2}{*}{Instance}} & \multicolumn{2}{c}{FunSearch}     & \multicolumn{2}{c}{EoH}& \multicolumn{2}{c}{ReEvo}         & \multicolumn{1}{c}{\multirow{2}{*}{EoH-S}} \\
\multicolumn{1}{c}{}     & \multicolumn{1}{c}{Top 1} & \multicolumn{1}{c}{Top 10} & \multicolumn{1}{c}{Top 1} & \multicolumn{1}{c}{Top 10} & \multicolumn{1}{c}{Top 1} & \multicolumn{1}{c}{Top 10} & \multicolumn{1}{c}{}  \\
\midrule
E-n22-k4     & 0.543& 0.260& \textbf{0.187} & \textbf{0.187}  & 0.306& 0.306& 0.263     \\
E-n23-k3     & 0.344& 0.290& 0.223& 0.139& \textbf{0.197} & \textbf{0.197}  & 0.202     \\
E-n30-k3     & 0.154& 0.146& 0.233& 0.168& 0.161& 0.134& \textbf{0.133}       \\
E-n33-k4     & 0.200& 0.164& 0.271& 0.153& 0.219& 0.219& \textbf{0.132}       \\
E-n51-k5     & 0.343& 0.253& 0.299& 0.233& 0.251& 0.168& \textbf{0.166}       \\
E-n76-k7     & 0.467& 0.467& 0.427& 0.396& 0.505& \textbf{0.294}  & 0.326     \\
E-n76-k8     & 0.533& 0.365& 0.401& 0.334& 0.320& \textbf{0.274}  & 0.284     \\
E-n76-k10    & 0.530& 0.447& 0.337& \textbf{0.272}  & 0.305& 0.298& 0.326     \\
E-n76-k14    & 0.434& 0.369& 0.330& 0.330& 0.245& 0.223& \textbf{0.212}       \\
E-n101-k8    & 0.605& 0.439& 0.426& 0.351& 0.363& 0.358& \textbf{0.290}       \\
E-n101-k14   & 0.479& 0.431& 0.545& 0.390& 0.486& 0.360& \textbf{0.324}       \\
\midrule
F-n45-k4     & 0.664& 0.509& \textbf{0.444} & \textbf{0.444}  & 0.690& 0.576& 0.507     \\
F-n72-k4     & 0.893& 0.605& 0.754& 0.630& 0.759& 0.558& \textbf{0.452}       \\
F-n135-k7    & 0.437& 0.393& 0.419& 0.406& 0.614& 0.508& \textbf{0.321}       \\
\midrule
M-n101-k10   & 0.535& 0.472& 0.486& 0.369& 0.533& 0.392& \textbf{0.225}       \\
M-n121-k7    & 0.428& 0.386& 0.406& 0.338& 0.467& 0.348& \textbf{0.231}       \\
M-n151-k12   & 0.496& 0.460& 0.478& 0.368& 0.443& 0.428& \textbf{0.358}       \\
M-n200-k16   & 0.463& 0.363& 0.467& 0.407& 0.384& 0.346& \textbf{0.320}       \\
M-n200-k17   & 0.462& 0.362& 0.466& 0.406& 0.383& \textbf{0.345}  & 0.363     \\
\midrule
P-n16-k8     & 0.110& 0.057& 0.034& 0.034& 0.028& \textbf{0.010}  & 0.033     \\
P-n19-k2     & 0.200& 0.200& 0.153& 0.153& 0.300& \textbf{0.079}  & 0.100     \\
P-n20-k2     & 0.063& 0.063& 0.140& 0.127& 0.239& \textbf{0.020}  & 0.063     \\
P-n21-k2     & 0.151& 0.151& 0.222& 0.100& 0.269& \textbf{0.065}  & 0.122     \\
P-n22-k2     & 0.205& 0.205& 0.194& 0.121& 0.244& 0.154& \textbf{0.098}       \\
P-n22-k8     & 0.330& 0.211& 0.292& \textbf{0.210}  & 0.385& 0.302& \textbf{0.210}       \\
P-n23-k8     & 0.212& 0.212& 0.200& 0.200& 0.169& 0.121& \textbf{0.087}       \\
P-n40-k5     & 0.380& 0.333& \textbf{0.203} & \textbf{0.203}  & 0.255& 0.210& 0.228     \\
P-n45-k5     & 0.340& 0.340& 0.342& \textbf{0.212}  & 0.453& 0.307& 0.245     \\
P-n50-k7     & 0.367& 0.328& 0.299& 0.249& \textbf{0.201} & \textbf{0.201}  & 0.237     \\
P-n50-k8     & 0.304& 0.304& 0.258& 0.258& 0.340& 0.237& \textbf{0.198}       \\
P-n50-k10    & 0.305& 0.299& 0.283& 0.267& 0.214& \textbf{0.193}  & 0.199     \\
P-n51-k10    & 0.347& 0.243& 0.377& 0.314& 0.342& 0.241& \textbf{0.200}       \\
P-n55-k7     & 0.289& 0.289& 0.285& 0.203& 0.296& 0.248& \textbf{0.181}       \\
P-n55-k10    & 0.304& 0.273& 0.303& 0.261& \textbf{0.186} & \textbf{0.186}  & 0.237     \\
P-n55-k15    & 0.241& 0.157& 0.287& 0.124& 0.152& 0.124& \textbf{0.104}       \\
P-n60-k10    & 0.358& 0.303& 0.323& 0.216& 0.335& 0.281& \textbf{0.166}       \\
P-n60-k15    & 0.218& 0.218& 0.396& 0.265& 0.331& 0.253& \textbf{0.145}       \\
P-n65-k10    & 0.332& 0.332& 0.275& \textbf{0.212}  & 0.364& 0.290& 0.256     \\
P-n70-k10    & 0.352& 0.352& 0.321& 0.214& 0.322& 0.170& \textbf{0.143}       \\
P-n76-k4     & 0.415& 0.388& 0.434& 0.314& \textbf{0.200} & \textbf{0.200}  & 0.215     \\
P-n76-k5     & 0.554& 0.422& 0.267& 0.221& 0.271& \textbf{0.195}  & 0.223     \\
P-n101-k4    & 0.516& 0.285& 0.386& 0.260& 0.320& 0.236& \textbf{0.190}       \\
\midrule
Average      & 0.379& 0.313& 0.330& 0.263& 0.330& 0.254& \textbf{0.222}\\
\bottomrule
\end{tabular}
}
\end{table*}

\begin{table*}[]
\tiny
\centering
\caption{Results on CVRPLib instances (Set X)}~\label{cvrplib_results_x}
\renewcommand{\arraystretch}{0.9}
\resizebox{0.75\textwidth}{!}{
\begin{tabular}{llllllll}
\toprule
\multicolumn{1}{c}{\multirow{2}{*}{Instance}} & \multicolumn{2}{c}{FunSearch}     & \multicolumn{2}{c}{EoH}& \multicolumn{2}{c}{ReEvo}         & \multicolumn{1}{c}{\multirow{2}{*}{EoH-S}} \\
\multicolumn{1}{c}{}     & \multicolumn{1}{c}{Top 1} & \multicolumn{1}{c}{Top 10} & \multicolumn{1}{c}{Top 1} & \multicolumn{1}{c}{Top 10} & \multicolumn{1}{c}{Top 1} & \multicolumn{1}{c}{Top 10} & \multicolumn{1}{c}{}  \\
\midrule
X-n101-k25   & 0.479& 0.379& 0.488& 0.353& 0.392& 0.312& \textbf{0.242}       \\
X-n106-k14   & 0.096& 0.091& 0.080& 0.079& 0.107& 0.070& \textbf{0.066}       \\
X-n110-k13   & 0.285& 0.218& 0.221& 0.214& 0.251& \textbf{0.147}  & 0.169     \\
X-n115-k10   & 0.493& 0.439& 0.487& 0.407& 0.503& \textbf{0.372}  & 0.394     \\
X-n120-k6    & 0.192& 0.189& 0.203& 0.184& 0.202& \textbf{0.183}  & 0.184     \\
X-n125-k30   & 0.226& 0.175& 0.173& 0.173& 0.214& 0.149& \textbf{0.139}       \\
X-n129-k18   & 0.241& 0.170& 0.199& 0.199& 0.154& \textbf{0.150}  & 0.167     \\
X-n134-k13   & 0.607& 0.469& 0.582& 0.415& 0.538& 0.369& \textbf{0.248}       \\
X-n139-k10   & 0.224& 0.212& 0.260& 0.230& 0.260& 0.226& \textbf{0.200}       \\
X-n143-k7    & 0.426& 0.391& 0.545& 0.397& 0.417& 0.366& \textbf{0.341}       \\
X-n148-k46   & 0.279& 0.275& 0.325& 0.229& 0.168& \textbf{0.139}  & 0.205     \\
X-n153-k22   & 0.429& 0.412& 0.427& 0.395& 0.399& 0.388& \textbf{0.285}       \\
X-n157-k13   & 0.098& 0.091& 0.068& 0.068& 0.225& 0.212& \textbf{0.063}       \\
X-n162-k11   & 0.215& 0.205& \textbf{0.139} & \textbf{0.139}  & 0.237& 0.147& 0.173     \\
X-n167-k10   & 0.236& \textbf{0.202}  & 0.275& 0.209& 0.233& 0.213& 0.205     \\
X-n172-k51   & 0.495& 0.401& 0.465& 0.395& 0.438& 0.364& \textbf{0.345}       \\
X-n176-k26   & 0.412& 0.327& 0.363& 0.327& 0.370& 0.314& \textbf{0.288}       \\
X-n181-k23   & 0.088& 0.081& 0.074& 0.069& 0.170& 0.170& \textbf{0.066}       \\
X-n186-k15   & 0.212& \textbf{0.154}  & 0.220& 0.220& 0.211& 0.187& 0.185     \\
X-n190-k8    & 0.177& 0.169& 0.177& 0.176& 0.153& \textbf{0.148}  & 0.181     \\
X-n195-k51   & 0.443& 0.309& 0.292& 0.292& 0.367& \textbf{0.268}  & 0.318     \\
X-n200-k36   & 0.190& 0.135& 0.188& 0.149& 0.168& 0.127& \textbf{0.103}       \\
X-n204-k19   & 0.232& 0.201& 0.229& 0.199& 0.201& \textbf{0.189}  & 0.199     \\
X-n209-k16   & 0.150& 0.127& 0.198& 0.155& 0.164& 0.140& \textbf{0.126}       \\
X-n214-k11   & 0.327& 0.270& 0.293& 0.290& 0.331& 0.291& \textbf{0.239}       \\
X-n219-k73   & 0.022& 0.018& 0.018& 0.018& 0.471& 0.467& \textbf{0.015}       \\
X-n223-k34   & 0.300& 0.289& 0.265& 0.265& 0.166& \textbf{0.113}  & 0.135     \\
X-n228-k23   & 0.383& 0.366& 0.368& 0.335& 0.296& \textbf{0.244}  & 0.270     \\
X-n233-k16   & 0.461& 0.461& 0.519& 0.420& 0.630& 0.441& \textbf{0.375}       \\
X-n237-k14   & 0.185& 0.185& 0.223& 0.198& 0.194& \textbf{0.166}  & 0.190     \\
X-n242-k48   & 0.174& 0.129& 0.145& 0.145& 0.114& \textbf{0.081}  & 0.092     \\
X-n247-k50   & 0.374& 0.355& 0.378& 0.376& 0.362& 0.327& \textbf{0.311}       \\
X-n251-k28   & \textbf{0.090} & \textbf{0.090}  & 0.115& 0.098& 0.150& 0.100& 0.102     \\
X-n256-k16   & 0.219& 0.202& 0.202& 0.184& 0.184& 0.162& \textbf{0.148}       \\
X-n261-k13   & 0.405& 0.345& 0.348& 0.348& 0.336& 0.324& \textbf{0.278}       \\
X-n266-k58   & 0.090& \textbf{0.064}  & 0.089& 0.078& 0.111& 0.076& 0.081     \\
X-n270-k35   & 0.147& 0.147& 0.203& 0.203& 0.231& \textbf{0.127}  & 0.132     \\
X-n275-k28   & 0.134& 0.132& 0.140& 0.135& 0.210& 0.196& \textbf{0.113}       \\
X-n280-k17   & 0.373& 0.263& 0.257& 0.257& 0.220& \textbf{0.205}  & 0.273     \\
X-n284-k15   & 0.249& 0.227& 0.279& \textbf{0.209}  & 0.242& 0.227& 0.228     \\
X-n289-k60   & 0.228& 0.219& 0.219& 0.219& 0.261& 0.216& \textbf{0.112}       \\
X-n294-k50   & 0.438& 0.376& 0.500& 0.427& 0.364& 0.344& \textbf{0.242}       \\
X-n298-k31   & 0.298& 0.298& 0.377& 0.294& 0.204& \textbf{0.175}  & 0.211     \\
\midrule
Average      & 0.275& 0.239& 0.270& 0.237& 0.270& 0.224& \textbf{0.196}     \\
\bottomrule
\end{tabular}
}
\end{table*}

\end{document}